\documentclass{article}

\usepackage[preprint]{neurips_2026}

\usepackage[utf8]{inputenc} 
\usepackage[T1]{fontenc}    
\usepackage{hyperref}       
\usepackage{url}            
\usepackage{booktabs}       
\usepackage{amsfonts}       
\usepackage{amsmath}
\usepackage{amssymb}
\usepackage{amsthm}
\usepackage{thmtools}
\usepackage{thm-restate}
\usepackage[nameinlink,noabbrev]{cleveref}
\usepackage{nicefrac}       
\usepackage{microtype}      
\usepackage{xcolor}         
\usepackage{graphicx}       
\usepackage{subcaption}     
\usepackage{xspace}
\newcommand{\adim}{d_a}
\newcommand{\edyn}{\varepsilon_{\mathrm{dyn}}}
\newcommand{\erew}{\varepsilon_{\mathrm{rew}}}
\newcommand{\Ndyn}{N_{\mathrm{dyn}}}
\newcommand{\Nrew}{N_{\mathrm{rew}}}
\newcommand{\cdyn}{c_{\mathrm{dyn}}}
\newcommand{\crew}{c_{\mathrm{rew}}}
\newcommand{\Lcomp}{L_{\mathrm{comp}}}
\newcommand{\splitcoef}{\frac{\gamma L_r(1+L_\pi)}{(1-\gamma)(1-\gamma L_f(1+L_\pi))}}
\newcommand{\mdp}{\textup{\textsc{mdp}}\xspace}
\newcommand{\mdps}{\textup{\textsc{mdp}}s\xspace}
\newcommand{\reinforce}{\textup{\textsc{reinforce}}\xspace}
\crefname{equation}{Equation}{Equations}
\Crefname{equation}{Equation}{Equations}
\crefname{theorem}{Theorem}{Theorems}
\Crefname{theorem}{Theorem}{Theorems}
\crefname{lemma}{Lemma}{Lemmas}
\Crefname{lemma}{Lemma}{Lemmas}
\crefname{proposition}{Proposition}{Propositions}
\Crefname{proposition}{Proposition}{Propositions}
\crefname{corollary}{Corollary}{Corollaries}
\Crefname{corollary}{Corollary}{Corollaries}
\crefname{definition}{Definition}{Definitions}
\Crefname{definition}{Definition}{Definitions}
\crefname{section}{Section}{Sections}
\Crefname{section}{Section}{Sections}

\newtheorem{corollary}{Corollary}
\theoremstyle{definition}
\newtheorem{definition}{Definition}
\theoremstyle{plain}

\title{On Training in Imagination}

\author{%
  Nadav Timor\thanks{Corresponding author: \texttt{nadav.timor@weizmann.ac.il}.} \\
  Weizmann Institute of Science
  \And
  Ravid Shwartz-Ziv \\
  New York University \\
  \And
  Micah Goldblum \\
  Columbia University
  \And
  Yann LeCun \\
  New York University \\
  AMI Labs
  \And
  David Harel \\
  Weizmann Institute of Science
}

\makeatletter
\let\MT@suspend@tagging\relax
\makeatother

\begin{document}

\maketitle

\begin{abstract}
State-of-the-art model-based reinforcement learning methods train policies on imagined rollouts. These rollouts are trajectories generated by a learned dynamics model and are scored by a learned reward model, but without querying the true environment during policy updates. We study this training paradigm by quantifying how errors in learned dynamics and reward models affect returns and policy optimization. First, we extend the analysis of \citet{asadi2018lipschitz} to MDPs with learned reward models, and derive the optimal sample allocation—the ratio of dynamics samples to reward samples that minimizes a bound on return error under power-law scaling assumptions. We identify lower Lipschitz constants of the learned dynamics, reward, and policy as a representation desideratum that tightens this bound, and we connect this perspective to the temporal-straightening objective of \citet{wang2026temporal}. Second, we examine how policy optimization with \reinforce tolerates noisy rewards, which are often cheaper to obtain. We show that zero-mean reward noise leaves the gradient estimator unbiased and adds at most a variance term that decreases with the number of rollouts. This introduces a practical tradeoff: given a fixed budget, should one buy more rollouts with cheaper but noisier rewards, or fewer rollouts with more expensive but less noisy rewards? We reduce this choice to a one-dimensional optimization problem and characterize the optimum.
\end{abstract}



\section{Introduction}

In \emph{training in imagination}, the policy is trained on rollouts generated by a learned dynamics model and scored by a learned reward model, with no environment interaction during the policy update step itself. Recent state-of-the-art instantiations include Dreamer~3 \citep{hafner2025mastering}, trained across diverse control tasks with a single configuration, and Dreamer~4 \citep{hafner2025training}, which extends the paradigm to long-horizon offline control. \citet{schrittwieser2020mastering} earlier instantiate a closely related paradigm in board games and Atari, learning dynamics, reward, and value jointly.

These recent results provide strong empirical evidence that training in imagination can be effective on challenging control tasks. Existing simulation-lemma-style bounds \citep{kearns2002near, asadi2018lipschitz}, however, do not assign independently controllable coefficients to dynamics-model and reward-model error, nor do they say how a sample budget should be split between dynamics samples and the typically more expensive reward samples (e.g., human preference labels in reinforcement learning from human feedback, or expert evaluation in robotics). The optimal trade-off between the two has not been theoretically characterized.

Four questions about this paradigm remain open. The first is error attribution: how much of the return gap comes from dynamics-model error versus reward-model error? The second is representation properties: what properties of the learned representations and the maps acting on them tighten the return-error bound? The third is budget allocation: given a fixed sample budget, how should it be split between dynamics transitions and reward annotations? The fourth is reward fidelity: how does \reinforce tolerate noisy or biased reward annotations, and when is it preferable to buy many cheap noisy annotations rather than fewer accurate ones?

\paragraph{Our contribution.}
This paper treats the learned reward model as a separate source of error with its own sample budget, distinct from the learned dynamics, and quantifies the resulting attribution, allocation, and noise-tolerance trade-offs under Lipschitz and power-law assumptions.
\begin{enumerate}
  \item \textbf{Error attribution.} \Cref{lem:decomposed} extends \citet{asadi2018lipschitz} by replacing the assumed ground-truth reward with a learned reward model, and gives a Lipschitz-based decomposition of the return gap with separable, independently controllable dynamics-error and reward-error coefficients.
  \item \textbf{Representation desiderata.} \Cref{cor:lipschitz_tightens_bound} shows that the dynamics-error coefficient in \cref{eq:decomposed_bound} is monotone non-decreasing in the Lipschitz constants $L_f, L_r, L_\pi$ of the learned dynamics, reward, and policy, identifying lower Lipschitz constants of the learned models as a representation desideratum. \Cref{prop:temporal_straightening} couples this perspective to the temporal-straightening objective of \citet{wang2026temporal}, upper-bounding its curvature loss by a function of the Lipschitz constant of the latent velocity map.
  \item \textbf{Budget allocation.} \Cref{thm:optimal_split}, under power-law error scaling for the dynamics and reward errors, gives a closed-form expression for the optimal ratio of dynamics samples to reward samples in terms of the power-law exponents, the per-sample costs, and the Lipschitz coefficient inherited from \cref{lem:decomposed}.
  \item \textbf{Reward fidelity.} \Cref{thm:noise} shows that the multi-trajectory \reinforce estimator under additive zero-mean reward noise is unbiased with bounded variance inflation; \Cref{cor:fidelity} reduces the optimal-fidelity allocation problem to a one-dimensional minimization in the per-rollout annotation cost; and \Cref{prop:biased_reward} formalizes systematic reward bias as a gradient bias that trajectory averaging cannot remove.
\end{enumerate}
Empirical evaluations of the assumptions and predictions of these results appear in \Cref{sec:lemma_calibration,sec:error_scaling,sec:allocation_evaluation}.

\paragraph{Notation.}
We write $\mathcal{M} = (\mathcal{S}, \mathcal{A}, f, r, \gamma)$ for a Markov decision process (\mdp) with state space $\mathcal{S} \subseteq \mathbb{R}^{d_s}$, action space $\mathcal{A} \subseteq \mathbb{R}^{\adim}$, deterministic dynamics $f$, reward $r$, and discount $\gamma \in [0,1)$.
Starting from an initial state $s_0$, a policy $\pi$ generates a trajectory by
$a_t = \pi(s_t)$ and $s_{t+1} = f(s_t, a_t)$.
We write $J(\pi, \mathcal{M}) := \sum_{t=0}^{\infty} \gamma^t r(s_t, a_t)$ for the discounted return of $\pi$ in $\mathcal{M}$.
Hats denote estimated quantities. In particular, $\hat{f}$ is the learned dynamics, $\hat{r}$ the learned reward, and $\hat{\mathcal{M}} = (\mathcal{S}, \mathcal{A}, \hat{f}, \hat{r}, \gamma)$ is the \mdp obtained by replacing $f$ and $r$ with $\hat{f}$ and $\hat{r}$.
$\edyn := \sup_{s,a} \|\hat{f}(s,a) - f(s,a)\|$ and $\erew := \sup_{s,a} |\hat{r}(s,a) - r(s,a)|$ are the worst-case model errors.
Throughout, $\|\cdot\|$ denotes the Euclidean norm.

\section{Related work}
\label{sec:related_work}

Training a policy on rollouts from a learned model of the environment dates back to the Dyna architecture of \citet{sutton1990integrated}, which interleaves real-environment transitions with updates on imagined transitions drawn from a learned dynamics model. Model-based policy optimization \citep{janner2019trust} uses short imagined rollouts from an ensemble dynamics model to augment policy updates. More recently, \citet{hafner2025mastering} train across diverse control tasks with a single learned world model and reward predictor, and \citet{hafner2025training} train agents inside scalable learned world models, calling this process ``training in imagination''. These works treat the learned dynamics model and the learned reward predictor as a single coupled object and tune it empirically; \cref{app:additional_related_work} reviews the broader latent-world-model lineage. Unlike the training-in-imagination lineage of \citet{hafner2025training}, this paper decomposes return error into separate, independently controllable dynamics-model and reward-model terms (\cref{lem:decomposed}), derives a closed-form split of a single sample budget between dynamics transitions and reward annotations (\cref{thm:optimal_split}), and characterizes the policy gradient inside this paradigm under both zero-mean noise and bias in the learned reward (\cref{thm:noise,prop:biased_reward}).

Simulation return-error bounds have a long history in reinforcement learning theory, beginning with the simulation lemma of \citet{kearns2002near}, which bounds the value gap between a true and approximate Markov decision process in terms of one-step transition and reward errors. Closest to our setting, \citet{asadi2018lipschitz} bound multi-step prediction error in model-based reinforcement learning under Lipschitz assumptions on the dynamics and policy, but they assume access to the ground-truth reward, so reward-model error never enters their bound. \Cref{app:additional_related_work} surveys subsequent refinements of these bounds and value-aware model learning. Unlike \citet{asadi2018lipschitz}, \cref{lem:decomposed} carries a learned reward model through the analysis and produces an explicit reward-error term alongside the dynamics-error term with independently controllable coefficients, which is the structural ingredient that makes a budget split between dynamics and reward data well-posed.

Allocating sample budget between dynamics samples and reward samples sits at the intersection of reward-aware data collection and neural scaling laws. Reward-free exploration separates a reward-agnostic data-collection phase from later reward-conditioned planning, with sample-complexity guarantees that hold uniformly over downstream reward functions \citep{jin2020reward}. Neural scaling laws fit power-law decays of loss in data and parameters \citep{kaplan2020scaling} and, in the compute-optimal regime, split a single training budget between model parameters and tokens \citep{hoffmann2022training}. \Cref{app:additional_related_work} discusses related work on active observation, simulation budget allocation, and scaling laws specific to reinforcement learning and world-model pre-training. Unlike \citet{hoffmann2022training}, \cref{thm:optimal_split} splits a single sample budget between two heterogeneous data streams---dynamics transitions and reward annotations---whose errors obey separately fitted power-law exponents, and yields a closed-form ratio in those exponents, the unit costs, and the Lipschitz coefficient inherited from \cref{lem:decomposed}.

Policy-gradient analysis under noisy rewards begins with the \reinforce estimator of \citet{williams1992simple}. A line of work studies robustness to reward corruption: \citet{zhang2021robust} analyze adversarial corruption in which an $\varepsilon$-fraction of episodes have their rewards or transitions arbitrarily modified and develop estimators with provable robustness guarantees. \citet{cai2025reinforcement} treat asymmetric verifier noise with false positives and false negatives in reinforcement learning with verifiable rewards. A parallel literature documents Goodhart-style overoptimization of learned reward models, in which the gap between proxy and gold rewards grows with optimization budget \citep{gao2023scaling}, a systematic failure mode rather than zero-mean reward noise. \Cref{app:additional_related_work} reviews the policy-gradient and variance-reduction lineage and the broader reward-modeling literature. Unlike the adversarial setting of \citet{zhang2021robust}, \cref{thm:noise,cor:fidelity} treat zero-mean i.i.d.\ reward noise as a per-rollout fidelity cost, which reduces annotation-budget allocation to a one-dimensional problem. \Cref{prop:biased_reward} addresses reward bias separately, showing that any non-zero reward-bias gradient survives trajectory averaging.

The choice of latent representation, and the regularity of the maps acting on it, governs how reliably an imagined rollout tracks reality. \citet{lecun2022path} advocates joint-embedding predictive architectures, in which prediction takes place in a learned latent space rather than at the level of raw observations. \citet{wang2026temporal} propose a temporal-straightening loss that encourages consecutive latent differences along a rollout to be parallel, so that long-horizon predictions in latent space follow a near-linear trajectory. On the regularization side, \citet{miyato2018spectral} introduce spectral normalization, which controls the Lipschitz constant of a neural network by normalizing the spectral norm of each weight matrix; \cref{app:additional_related_work} discusses related joint-embedding instantiations and Lipschitz-regularization mechanisms. These representation-learning and Lipschitz-regularization proposals are motivated by stability or representational quality, and their connection to long-horizon policy value is left implicit. Unlike \citet{wang2026temporal} and \citet{miyato2018spectral}, \cref{cor:lipschitz_tightens_bound,prop:temporal_straightening} couple these representation-learning desiderata to an explicit return-error coefficient, showing that the dynamics-error coefficient is monotone in the Lipschitz constants of the learned dynamics, reward, and policy and that the temporal-straightening loss is upper-bounded by a function of the latent-velocity-map Lipschitz constant.

\section{Properties of representations for training in imagination}

What makes a representation of the system useful for training in imagination?
\citet{lecun2022path} hypothesized that designing the dynamics, reward, and policy models to operate on latent states $z_t$ that capture a higher-level representation of the system---in place of raw observations $s_t$ and past actions $a_t$---may improve prediction and planning across different time horizons.
However, what properties such a representation should have has remained an open question.
\Cref{lem:decomposed} and \Cref{cor:lipschitz_tightens_bound} give one answer to this open question: representations that lower the Lipschitz constants of the learned models tighten our bound on return error.

Similar to \citet{asadi2018lipschitz}, we assume that the dynamics $f$, reward $r$, and policy $\pi$ satisfy the following Lipschitz conditions:
\begin{itemize}
  \item $f$ is $L_f$-Lipschitz: $\|f(s,a) - f(s',a')\| \leq L_f(\|s - s'\| + \|a - a'\|)$ for all $s,s',a,a'$.
  \item $r$ is $L_r$-Lipschitz: $|r(s,a) - r(s',a')| \leq L_r(\|s - s'\| + \|a - a'\|)$ for all $s,s',a,a'$.
  \item $\pi$ is $L_\pi$-Lipschitz: $\|\pi(s) - \pi(s')\| \leq L_\pi\|s - s'\|$ for all $s,s'$.
\end{itemize}

\begin{restatable}[Simulation error decomposition]{lemma}{decomposedlemma}\label{lem:decomposed}
Let $\mathcal{M}, \hat{\mathcal{M}}$ be \mdps sharing $(\mathcal{S}, \mathcal{A}, \gamma)$ with deterministic dynamics $f, \hat{f}$ and rewards $r, \hat{r}$.
Assume $\gamma L_f(1 + L_\pi) < 1$.
Then for any $L_\pi$-Lipschitz policy $\pi$,
\begin{equation}
\label{eq:decomposed_bound}
\left| J(\pi, \mathcal{M}) - J(\pi, \hat{\mathcal{M}}) \right|
\leq
\frac{1}{1-\gamma}\,\erew
+
\splitcoef\,\edyn.
\end{equation}
\end{restatable}

\begin{proof}
See \Cref{app:proof_decomposed}.
\end{proof}

\begin{corollary}[Lipschitz constants control the dynamics-error coefficient]\label{cor:lipschitz_tightens_bound}
Under the hypothesis $\gamma L_f(1+L_\pi)<1$, the coefficient $\splitcoef$ of $\edyn$ in \cref{eq:decomposed_bound} is non-decreasing in each of $L_f$, $L_r$, and $L_\pi$. Hence, at fixed $\edyn$ and $\erew$, lowering any of $L_f, L_r, L_\pi$ tightens the bound in \cref{eq:decomposed_bound} on return error.
\end{corollary}

\Cref{cor:lipschitz_tightens_bound} formalizes this: representations that lower the Lipschitz constants of the learned models $L_f, L_r, L_\pi$ tighten the bound in \cref{eq:decomposed_bound} on return error at fixed $\edyn, \erew$.

A simple implementation may define the latent state $z$ as an encoding of the current observation $z = \phi(s)$.
The Lipschitz constants are then defined by comparing the outputs of the learned models at nearby latent states, rather than at nearby raw observations. For example, the Lipschitz constant of the dynamics model is then $\|f(\phi(s),a)-f(\phi(s'),a')\| \leq L_f\bigl(\|\phi(s)-\phi(s')\|+\|a-a'\|\bigr)$. The reward and policy maps are handled analogously, with absolute value replacing the output norm for scalar rewards.
Higher-order representations are also possible: latent states $z_t$ may encode the full trajectory $(s_0, a_0, \ldots, s_t, a_{t-1})$, states or activations of other learned models, and so on.

This Lipschitz perspective also connects to the temporal straightening objective of \citet{wang2026temporal}, which compares consecutive latent differences generated by the dynamics model and encourages these differences to point in the same direction.
To make the connection explicit, let $\mathcal{Z}$ denote the latent state space and write $z_t = \phi(s_t)$ for the latent state associated with observation $s_t$.
Following \citet{wang2026temporal}, we write $f$ for the learned dynamics model on $\mathcal{Z}$, so that $z_{t+1}=f(z_t)$ along a latent rollout.
Thus, in this discussion of temporal straightening, $f$ no longer denotes the observation-level dynamics on $\mathcal{S}$.
\begin{definition}[temporal straightening curvature loss, adapted from \citet{wang2026temporal}]
\label{def:temporal_straightening_loss}
Define the latent velocity map $v(z):=f(z)-z$.
For a latent rollout $z_{t+1}=f(z_t)$ and any $t$ such that $v(z_t)\neq 0$ and $v(z_{t+1})\neq 0$, define
\begin{equation*}
\mathcal{L}_{\mathrm{curv}}(t)
:= 1 - \frac{v(z_t)^\top v(z_{t+1})}{\|v(z_t)\|\|v(z_{t+1})\|}.
\end{equation*}
\end{definition}

On observed transitions, $v(\phi(s_t))$ approximates $\phi(s_{t+1})-\phi(s_t)$, so changes in $v$ along a latent rollout measure how smoothly the latent state moves along the trajectory.
The temporal straightening objective maximizes the cosine similarity between $v(z_t)$ and $v(z_{t+1})$, or equivalently minimizes the curvature loss in \Cref{def:temporal_straightening_loss}.
The relevant Lipschitz quantity is not the Lipschitz constant of $f$ itself, but the Lipschitz constant of the latent velocity map $v$.

\begin{restatable}[temporal straightening from a Lipschitz latent velocity map]{proposition}{straighteningthm}\label{prop:temporal_straightening}
Let $z_{t+1}=f(z_t)$ be a latent rollout generated by the dynamics model, and define $v(z):=f(z)-z$.
Assume $v$ is $\varepsilon$-Lipschitz on the latent states visited by the rollout, with $0<\varepsilon<1$.
For any $t$ such that $v(z_t)\neq 0$ and $v(z_{t+1})\neq 0$, the temporal straightening curvature loss from \Cref{def:temporal_straightening_loss} satisfies
\begin{equation}
\label{eq:temporal_straightening_bound}
\mathcal{L}_{\mathrm{curv}}(t)
\leq \frac{\varepsilon^2}{2(1-\varepsilon)}.
\end{equation}
\end{restatable}

\begin{proof}
See \Cref{app:proof_temporal_straightening}.
\end{proof}

\Cref{prop:temporal_straightening} shows that making the latent velocity map slowly varying makes consecutive latent differences nearly parallel, which directly lowers the temporal straightening loss.
Thus, minimizing the Lipschitz constant of $v$ minimizes the upper bound in \cref{eq:temporal_straightening_bound}.

Two caveats apply.
First, a representation that lowers the Lipschitz constants might also increase $\edyn$. The bound in \Cref{lem:decomposed} only tightens if the decrease in the dynamics-error coefficient outweighs the increase in $\edyn$.
Second, that bound assumes $\gamma L_f(1+L_\pi) < 1$, and the behavior when $\gamma L_f(1+L_\pi) \geq 1$ remains an open question for future work.

\subsection{Numerical illustration of \Cref{lem:decomposed}}
\label{sec:lemma_calibration}
We empirically test the inequality in \cref{eq:decomposed_bound}.
For each test configuration, define the realized return gap $\mathrm{LHS} := |J(\pi, \mathcal{M}) - J(\pi, \hat{\mathcal{M}})|$ and the bound's right-hand side $\mathrm{RHS} := (1-\gamma)^{-1}\,\erew + \splitcoef\,\edyn$, evaluated with the configuration's analytical Lipschitz constants $L_f, L_r, L_\pi$, discount factor $\gamma$, and realized per-step errors $\edyn, \erew$. We report the ratio $R := \mathrm{LHS}/\mathrm{RHS}$. The bound holds whenever $R \le 1$, and smaller $R$ indicates a looser bound.
Each configuration consists of an \mdp, an $L_\pi$-Lipschitz evaluation policy, and a perturbed model pair with realized per-step errors $\edyn, \erew$. We test on two benchmarks. The synthetic benchmark uses globally Lipschitz $f$ and $r$, so the hypotheses of \Cref{lem:decomposed} hold by construction. The Linear--Quadratic--Gaussian (LQG) benchmark uses a quadratic reward that is only locally Lipschitz, testing \Cref{lem:decomposed} on a bounded operating domain. \Cref{app:lemma_calibration_details} describes the per-configuration construction.
Across $n = 525$ configurations ($150$ synthetic, $375$ LQG) the bound holds on every configuration. The per-benchmark medians are $R^{\mathrm{med}}_{\mathrm{synth}} = 0.0035$ and $R^{\mathrm{med}}_{\mathrm{LQG}} = 0.034$, the pooled median is $0.015$, and the per-benchmark maxima are $R^{\max}_{\mathrm{synth}} = 0.9995$ and $R^{\max}_{\mathrm{LQG}} = 0.999$. For per-benchmark empirical CDF (ECDF) shapes and full implementation details, see \Cref{app:lemma_calibration_details}.
Every configuration satisfies the bound, but the typical bound is loose: at the median it overshoots by $1/R^{\mathrm{med}}_{\mathrm{LQG}} \approx 29\times$ on LQG and $1/R^{\mathrm{med}}_{\mathrm{synth}} \approx 286\times$ on the synthetic benchmark (factor $\approx 65\times$ at the pooled median). Both benchmarks use deterministic dynamics; extending the calibration to stochastic dynamics or unbounded operating domains would require revisiting the assumptions of \Cref{lem:decomposed} and is left to future work.

\section{The optimal sample allocation to minimize return error}
\label{sec:optimal_split}

We distinguish between two types of samples: dynamics-transition samples and reward samples. A dynamics-transition sample $(s,a,f(s,a))$ consists of
a state $s$, an action $a$, and the resulting next state $f(s,a)$.
A reward sample $(s,a,r(s,a))$ consists of a state, an action, and the resulting reward $r(s,a)$. Let $\Ndyn$ and $\Nrew$ denote the numbers of
dynamics-transition and reward samples, and let $\cdyn$ and $\crew$ denote their per-sample costs. These costs may reflect any incurred costs, including environment interaction, annotation, and training on the samples.
We reuse the symbols $\edyn$ and $\erew$ to denote the error
levels achievable when the sample counts are $\Ndyn$ and $\Nrew$, according to standard
power laws
\begin{equation}
\label{eq:error_scaling}
\edyn(\Ndyn) = A_d \cdot \Ndyn^{-\alpha}, \qquad
\erew(\Nrew) = A_r \cdot \Nrew^{-\beta}
\end{equation}
where $\alpha, \beta > 0$ are exponents fit from data, and $A_d, A_r$ are constants.
For each budget $B = \cdyn \Ndyn + \crew \Nrew$, let $(\Ndyn^*, \Nrew^*)$ denote a minimizer of the bound in \cref{eq:decomposed_bound}, and let $\edyn^* := \edyn(\Ndyn^*)$ and $\erew^* := \erew(\Nrew^*)$ be the dynamics and reward errors at those minimizing sample counts. We study the optimal ratio of dynamics samples to reward samples, $\frac{\Ndyn^*}{\Nrew^*}$, as $B \to \infty$.

\begin{restatable}[Optimal dynamics-to-reward sample ratio]{theorem}{optimalsplitthm}\label{thm:optimal_split}
If \cref{eq:error_scaling} holds with exponents $\alpha, \beta > 0$ and the bound of \Cref{lem:decomposed} holds, then the minimizing sample counts $(\Ndyn^*, \Nrew^*)$ under the constraint $\cdyn \Ndyn + \crew \Nrew = B$ satisfy
\begin{equation}
\label{eq:optimal_ratio}
\frac{\Ndyn^*}{\Nrew^*}
:= \frac{\alpha}{\beta} \cdot \frac{\gamma L_r(1+L_\pi)}{1-\gamma L_f(1+L_\pi)} \cdot \frac{\crew}{\cdyn} \cdot \frac{\edyn^*}{\erew^*}.
\end{equation}
\end{restatable}

\begin{proof}
See \Cref{app:proof_optimal_split}.
\end{proof}

\Cref{thm:optimal_split} gives the optimal sample ratio $\Ndyn^*/\Nrew^*$ in terms of the error ratio $\edyn^*/\erew^*$, and shows that these two ratios are proportional. The multiplier in \cref{eq:optimal_ratio} uses the global Lipschitz constants $L_f, L_r$ assumed by \Cref{lem:decomposed}, so it is an upper bound rather than an equality. \Cref{sec:allocation_evaluation} measures, on the same configurations, how much smaller the realized value-function sensitivities are than these global constants, and how close the resulting prediction of $\Ndyn^*/\Nrew^*$ becomes. For practitioners who design systems that train policies in imagination, changes in sample costs or planning horizon affect the optimal sample ratio as follows.

\paragraph{Sample costs.}
Costs enter through $\cdyn\Ndyn + \crew\Nrew = B$ and the fitted error curves in \cref{eq:error_scaling}; in \cref{eq:optimal_ratio}, they appear as $\crew/\cdyn$.
Lowering $\crew$ lowers this factor, so $\Ndyn^*/\Nrew^*$ decreases and the allocation shifts toward reward samples.
Lowering $\cdyn$ raises the factor, so $\Ndyn^*/\Nrew^*$ increases and the allocation shifts toward dynamics-transition samples.
However, since the optimal sample ratio also includes $\edyn^*/\erew^*$, these statements give the direction of the cost effect rather than a complete allocation rule by themselves.

\paragraph{Planning horizon.}
The discounting effect is the familiar one: smaller $\gamma$ gives less weight to later rewards.
As $\gamma$ decreases, \cref{eq:decomposed_bound} becomes tighter, and the dynamics multiplier in \cref{eq:optimal_ratio}, $\gamma L_r(1+L_\pi)/(1-\gamma L_f(1+L_\pi))$, decreases.
The optimal allocation therefore shifts toward reward samples, as expected for a shorter-horizon objective in which dynamics errors have fewer discounted steps to affect future rewards.
%
A similar effect arises from lowering the Lipschitz constants of the learned models. By \Cref{cor:lipschitz_tightens_bound}, lowering any of $L_f, L_r, L_\pi$ tightens \cref{eq:decomposed_bound} and decreases the dynamics multiplier in \cref{eq:optimal_ratio}. The optimal allocation therefore shifts toward reward samples, since dynamics errors then compound less strongly along rollouts.

\subsection{Numerical experiments: scaling of dynamics and reward errors}
\label{sec:error_scaling}

We estimate how fast each model learns with more data and whether the dynamics and reward models learn at the same rate.
To this end, we conduct an experiment that tests the power-law scaling assumptions in \cref{eq:error_scaling} and estimates the exponents $\alpha$ and $\beta$ for a particular architecture and environment. In this experiment, we estimate the dynamics and reward errors as a function of the number of training samples $\Ndyn$ and $\Nrew$, respectively, and fit the standard power-law scaling laws from \Cref{eq:error_scaling}. \Cref{fig:error_scaling} shows the fitted scaling laws, which are consistent with the power-law assumptions in \Cref{thm:optimal_split}.
It measures $\edyn$ and $\erew$ on a fixed held-out set $\mathcal{D}_{\mathrm{val}}$ of transitions $(s, a, s', r)$ in a synthetic continuous-control environment whose dynamics and reward function are defined by frozen, randomly-initialized 2-layer ReLU MLPs, and whose students $\hat f$ and $\hat r$ use the same architecture as the teachers. Specifically, $\edyn(\Ndyn) := \frac{1}{|\mathcal{D}_{\mathrm{val}}|} \sum_{(s, a, s', r) \in \mathcal{D}_{\mathrm{val}}} \|\hat f(s, a) - s'\|^2$ and $\erew(\Nrew) := \frac{1}{|\mathcal{D}_{\mathrm{val}}|} \sum_{(s, a, s', r) \in \mathcal{D}_{\mathrm{val}}} (\hat r(s, a) - r)^2$, where $\hat f$ and $\hat r$ are trained on $\Ndyn$ dynamics-transition samples and $\Nrew$ reward samples drawn independently of $\mathcal{D}_{\mathrm{val}}$. We use $\edyn$ and $\erew$ as practical surrogates for the sup-norm errors in \Cref{lem:decomposed}.
The fitted laws are $\edyn(\Ndyn) = 0.34\, \Ndyn^{-0.11}$ with $R^2 = 0.954$ and $\erew(\Nrew) = 90.4\, \Nrew^{-0.96}$ with $R^2 = 0.997$. For bootstrap standard errors, 95\% bootstrap confidence intervals, and full implementation details, see \Cref{app:error_scaling_details}.
The exponent ratio $0.96 / 0.11 \approx 9$ is the central empirical observation: reward error decays nearly an order of magnitude faster per decade of training data than dynamics error. This ratio of exponents is consistent with $\hat r$ predicting a scalar while $\hat f$ predicts a $d_s$-dimensional next state, and its exact value depends on the dimensions and architectures used here.

\begin{figure}[t]
  \centering
  \includegraphics[width=0.78\linewidth]{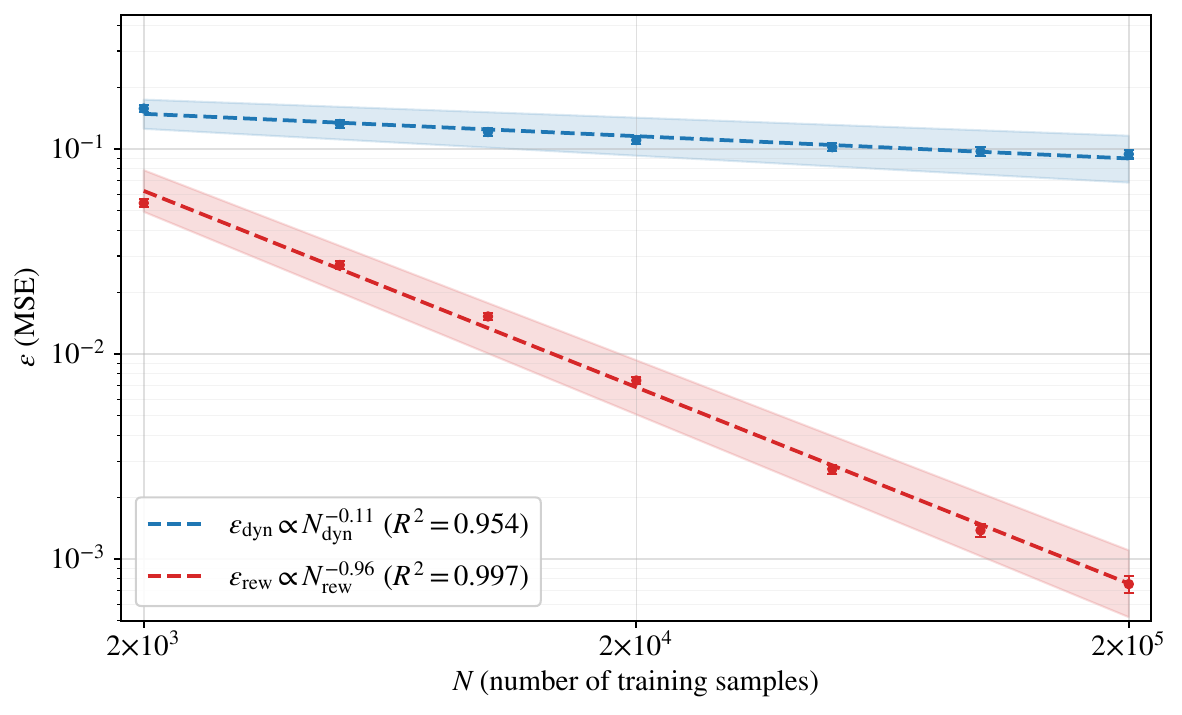}
  \caption{Dynamics and reward errors obey distinct power-law scaling laws.
  Reward error decays roughly $9\times$ faster per decade of training data than dynamics error ($\approx \frac{0.96}{0.11}$).
  \Cref{thm:optimal_split} uses this ratio to characterize the optimal fraction of transitions that should carry a reward annotation.}
  \label{fig:error_scaling}
\end{figure}

\subsection{Empirical evaluation of \cref{eq:optimal_ratio}}
\label{sec:allocation_evaluation}

\Cref{eq:optimal_ratio} predicts $\Ndyn^*/\Nrew^*$ as the product of two factors: the proportionality $\Ndyn^*/\Nrew^* \propto \edyn^*/\erew^*$ and a multiplier that depends on $\gamma$, $L_f$, $L_r$, $L_\pi$, $\crew/\cdyn$, and $\alpha/\beta$. We test the two factors separately: whether the proportionality holds when the multiplier is replaced by realized value-function sensitivities, and how loose the multiplier is when instantiated with the global Lipschitz constants assumed by \Cref{lem:decomposed}.
Define the log-ratio residual $\ell := \log(\Ndyn^*/\Nrew^*) - \log(\edyn^*/\erew^*)$, so $|\ell| \le \log 3$ means predicted and realized ratios agree within a factor of $3$. The realized value-function sensitivities $S_f, S_r$ and the per-configuration construction are defined in \Cref{app:allocation_evaluation_details}.
Replacing the global constants $L_f, L_r$ with the realized sensitivities $S_f, S_r$ recovers the predicted ratio. We test this recovery on configurations whose value function $V^\pi$ has a known parametric form: linear, $\tanh$, $\sin$, and a separate quadratic-value control group that isolates the looseness from using $\sup$-norm overestimates in place of the realized sensitivities. The linear configurations achieve median $|\ell| = 0$ and Spearman rank correlation $\rho = 1$ between predicted and realized ratios; this case is an algebraic consistency check rather than an empirical test, since a linear $V^\pi$ forces the realized sensitivities to match the analytical ratio. The empirical content comes from the nonlinear configurations: $\tanh$ and $\sin$ give median $|\ell| = 0.054$. On the control group, substituting $\sup$-norm overestimates for the realized sensitivities inflates the median residual by roughly an order of magnitude. \Cref{fig:allocation_form} plots predicted against realized sample ratios $\Ndyn^*/\Nrew^*$ for each group.
Instantiating the multiplier with the analytical global Lipschitz constants of \Cref{lem:decomposed}, by contrast, overshoots realized ratios by about three orders of magnitude. On the LQG configurations every configuration's $\ell$ is positive, with median $\ell = 7.585$ (a factor $\exp(7.585) \approx 1968$ between predicted and realized ratios). \Cref{fig:lipschitz_slack} shows the per-configuration distribution of $\ell$. The contraction regime, sample sizes, statistical methodology, and a separate ratio isolating the dynamics-coefficient contribution to the multiplier are reported in \Cref{app:allocation_evaluation_details}.

\begin{figure}[t]
  \centering
  \includegraphics[width=0.85\linewidth]{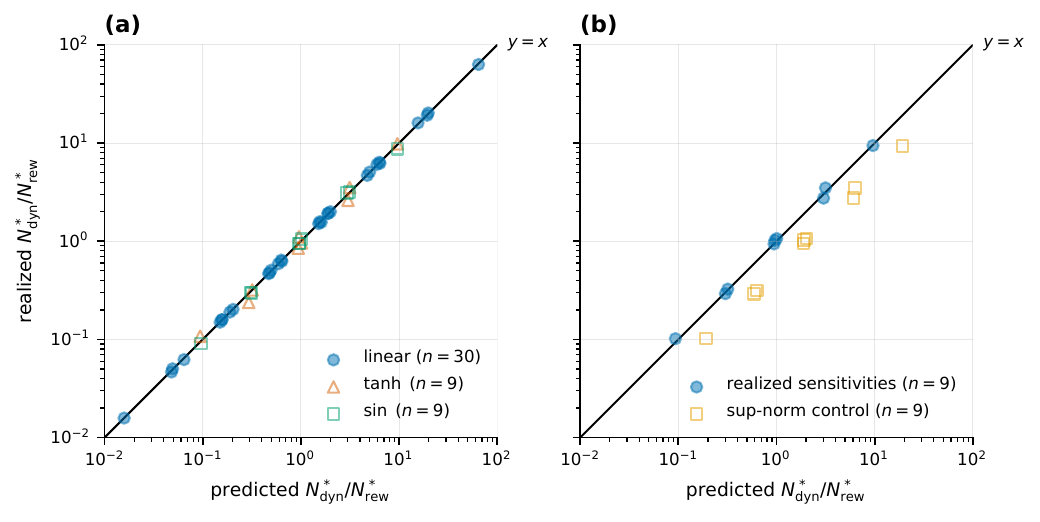}
  \caption{Predicted and realized sample ratios $\Ndyn^*/\Nrew^*$ agree when global Lipschitz constants are replaced by realized value-function sensitivities $S_f, S_r$. The linear, $\tanh$, and $\sin$ groups in panel~(a) concentrate on the diagonal $y=x$, while the $\sup$-norm control in panel~(b) lies systematically below it---using the global $\sup$-norm in place of the realized sensitivities reintroduces the looseness shown in \Cref{fig:lipschitz_slack}. See \Cref{app:allocation_evaluation_details} for plotting and methodology details.}
  \label{fig:allocation_form}
\end{figure}

The proportionality $\Ndyn^*/\Nrew^* \propto \edyn^*/\erew^*$ in \cref{eq:optimal_ratio} therefore carries the predictive content, while the multiplier built from global Lipschitz constants is loose at the order-of-magnitude scale. Recovering the predicted ratio requires evaluating $V^\pi$ at perturbed models (the construction in \Cref{app:allocation_evaluation_details}), which may be impractical at scale.

\section{Cost-effective learning from noisy rewards}
\label{sec:noisy_rewards}

The reward-noise analysis in this section is independent of the allocation results in \Cref{sec:optimal_split}: the unbiasedness statement for \reinforce under zero-mean reward noise applies to any stationary MDP satisfying the standard policy-gradient assumptions, regardless of how dynamics and reward samples are allocated.

\noindent We study policy optimization under noisy rewards through the
classical \reinforce estimator introduced by
\citet{williams1992simple}. In this section let $\pi_\theta$ be a
differentiable policy, fix a finite horizon $H$, and fix a discount
factor $\gamma \in [0,1)$. Define
$J_H(\pi, \mathcal{M}) := \mathbb{E}[\sum_{t=0}^{H-1} \gamma^t r(s_t, a_t)]$,
where the expectation is over trajectories generated by executing $\pi$ in $\mathcal{M}$.
Define the discounted cumulative reward from time $t$ onward as
$G_t := \sum_{t'=t}^{H-1} \gamma^{t'-t} r(s_{t'}, a_{t'})$,
so that $G_0$ is the discounted return of the full trajectory and
$J_H(\pi, \mathcal{M}) = \mathbb{E}[G_0]$. Define the finite-horizon
policy gradient by
$g_H := \nabla_\theta J_H(\pi_\theta, \mathcal{M})$.
Define the noisy counterpart of $G_t$ as
$\hat{G}_t := \sum_{t'=t}^{H-1} \gamma^{t'-t}\hat{r}_{t'}$, where $\hat{r}_t$ denotes the observed reward at time $t$.
The \reinforce estimator of $g_H$ computed from a single trajectory is
$\hat{g}^{(1)} := \sum_{t=0}^{H-1} \nabla_\theta \log \pi_\theta(a_t \mid s_t)\,\hat{G}_t$.
To estimate $g_H$, sample $K \geq 1$
independent trajectories by executing $\pi_\theta$ in $\mathcal{M}$.
For each $k \in \{1,\ldots,K\}$, let $\hat{g}^{(k)}$ denote the
\reinforce estimator computed on the $k$th trajectory, and define
$\hat{g} := \frac{1}{K}\sum_{k=1}^K \hat{g}^{(k)}$.
The estimators $\hat{g}^{(1)},\ldots,\hat{g}^{(K)}$, $\hat{g}$, and their
noise-free counterparts are vector-valued. To measure their dispersion with a
single scalar, we use the natural generalization of scalar variance obtained by
adding coordinate variances: for $z=(z_1,\ldots,z_d)$, write
$\mathrm{Var}[z] := \sum_{i=1}^d \mathrm{Var}[z_i]$.

\begin{restatable}[Finite-horizon \reinforce under noisy rewards]{theorem}{noisethm}\label{thm:noise}
Assume
$W_H^2 := \mathbb{E}\bigl[\max_{0 \leq t \leq H-1}\|\nabla_\theta \log \pi_\theta(a_t \mid s_t)\|^2\bigr] < \infty$.
Suppose the rewards are observed with additive noise,
$\hat{r}_t = r(s_t, a_t) + \eta_t$, where the noise variables $\eta_t$ are
i.i.d.\ with $\mathbb{E}[\eta_t] = 0$ and
$\mathrm{Var}[\eta_t] = \sigma_\eta^2 < \infty$, and each $\eta_t$ is
independent of the state-action history $((s_0,a_0),\ldots,(s_t,a_t))$.
Then $\hat{g}$ satisfies
\begin{equation}
\mathbb{E}[\hat{g}] = g_H, \qquad
\mathrm{Var}[\hat{g}] \leq \mathrm{Var}[\hat{g}]_{\eta \equiv 0}
+ \frac{\sigma_\eta^2 H W_H^2}{K(1-\gamma)^2}.
\end{equation}
where $\mathrm{Var}[\hat{g}]_{\eta \equiv 0}$ denotes the variance of the same
estimator when the rewards are noise-free.
\end{restatable}

\begin{proof}
See \Cref{app:proof_noise}.
\end{proof}

In many settings, practitioners can pay more per reward annotation to obtain less noisy rewards, for example by averaging multiple annotators or using a more careful annotation pipeline. We use \Cref{thm:noise} to ask how a fixed budget for reward annotations should be split between fidelity (lower noise per annotation) and quantity (more rollouts at higher noise).

Let $c > 0$ denote the per-rollout cost of acquiring reward annotations along that rollout, and define $\sigma_\eta^2(c) := \mathrm{Var}[\eta_t \mid \text{per-rollout annotation cost equals } c] \in [0, \infty)$ as the variance of the reward noise $\eta_t$ from \Cref{thm:noise} when reward annotations are acquired at per-rollout cost $c$.
We assume that $\sigma_\eta^2 \colon (0,\infty) \to [0,\infty)$ is measurable. Given a budget $B > 0$ for reward annotations, the number of independent rollouts that fit in the budget at fidelity $c$ is $K = B/c$.

\begin{restatable}[Optimal fidelity for noise-induced variance]{corollary}{fidelitythm}\label{cor:fidelity}
Define $\Phi(c) := c\,\sigma_\eta^2(c)$. Under the assumptions of \Cref{thm:noise} with $\mathrm{Var}[\eta_t] = \sigma_\eta^2(c)$ and $K = B/c$, the upper bound on the noise-induced excess variance from \Cref{thm:noise} equals $\Phi(c)\,H\,W_H^2 / (B\,(1-\gamma)^2)$.
Consequently, any $c^* \in \operatorname*{arg\,min}_{c > 0} \Phi(c)$ minimizes this upper bound over $c > 0$.
\end{restatable}

\begin{proof}
Substituting $K = B/c$ into the upper bound from \Cref{thm:noise} gives $\sigma_\eta^2(c)\,H\,W_H^2 / (K\,(1-\gamma)^2) = \Phi(c)\,H\,W_H^2 / (B\,(1-\gamma)^2)$. The prefactor $H\,W_H^2 / (B\,(1-\gamma)^2)$ is non-negative and does not depend on $c$, so the upper bound and $\Phi(c)$ share their minimizers over $c > 0$.
\end{proof}

\Cref{cor:fidelity} reduces the choice of fidelity for reward annotations to minimizing $\Phi(c)$. We illustrate the consequences with three examples of how the noise variance $\sigma_\eta^2(c)$ depends on the cost $c$, summarized in \Cref{fig:fidelity_examples}.

\paragraph{Power-law fidelity.}
Suppose $\sigma_\eta^2(c) = A\,c^{-p}$ for constants $A, p > 0$, so that each multiplicative increase in cost yields a fixed multiplicative reduction in noise variance. Then $\Phi(c) = A\,c^{1-p}$, which is strictly decreasing in $c$ when $p > 1$, strictly increasing in $c$ when $p < 1$, and constant when $p = 1$. The exponent $p = 1$ therefore separates two regimes: when $p > 1$, the bound is minimized by paying for the highest-fidelity annotations available; when $p < 1$, the bound is minimized by paying as little as possible per annotation and using the saved budget for additional rollouts.

\paragraph{Bounded fidelity.}
Suppose $\sigma_\eta^2(c) = \sigma_0^2\,(1 - c/c_{\max})$ on $(0, c_{\max}]$ for constants $\sigma_0^2, c_{\max} > 0$, modeling a setting in which annotations have variance $\sigma_0^2$ as $c \to 0$ and become noise-free at the finite cost $c_{\max}$. Then $\Phi(c) = \sigma_0^2\,c\,(1 - c/c_{\max})$ is a downward parabola in $c$ that is maximized at $c = c_{\max}/2$, attains the value $0$ at $c = c_{\max}$, and approaches $0$ as $c \to 0$. Both extremes of the cost range therefore minimize the bound, and intermediate fidelities are strictly worse.

\paragraph{Irreducible noise floor.}
Suppose $\sigma_\eta^2(c) = \sigma_{\mathrm{floor}}^2 + A/c$ for constants $\sigma_{\mathrm{floor}}^2, A > 0$, modeling a setting in which no level of spending can drive the noise variance below $\sigma_{\mathrm{floor}}^2$. Then $\Phi(c) = \sigma_{\mathrm{floor}}^2\,c + A$ is strictly increasing in $c$, so the bound is minimized by $c \to 0$: when reward noise has a floor that money cannot remove, the optimal allocation buys the cheapest annotations and relies on the $1/K$ factor in \Cref{thm:noise} to reduce variance.

Together, these cases show that the shape of $\Phi(c)$ determines whether the budget should be spent on fidelity, on quantity, or on one of the two extremes.

The preceding results assume the reward noise is zero-mean. We note that this assumption is essential: averaging over more trajectories cannot remove a systematic reward bias.

\begin{restatable}[Finite-horizon \reinforce under biased rewards]{proposition}{biasedrewardthm}\label{prop:biased_reward}
Let $b:\mathcal{S}\times\mathcal{A}\to\mathbb{R}$ be a reward bias function,
define $\tilde{r}(s,a) := r(s,a) + b(s,a)$, and define
$\tilde{\mathcal{M}} := (\mathcal{S}, \mathcal{A}, f, \tilde{r}, \gamma)$.
For trajectories sampled by executing $\pi_\theta$ under the true dynamics $f$,
let $\tilde{g}^{(1)},\ldots,\tilde{g}^{(K)}$ be independent \reinforce
estimators computed using the biased rewards $\tilde{r}(s_t,a_t)$, and define
$\tilde{g} := \frac{1}{K}\sum_{k=1}^K \tilde{g}^{(k)}$.
Define $B_H(\theta) :=
\mathbb{E}\!\left[\sum_{t=0}^{H-1}\gamma^t b(s_t,a_t)\right]$, where the
expectation is over trajectories generated by executing $\pi_\theta$ under the
true dynamics $f$. Then
\begin{equation}
\label{eq:biased_reward_mean}
\mathbb{E}[\tilde{g}]
= \nabla_\theta J_H(\pi_\theta,\tilde{\mathcal{M}})
= g_H + \nabla_\theta B_H(\theta).
\end{equation}
If $\mathrm{Var}[\tilde{g}^{(1)}] < \infty$, then
\begin{equation}
\label{eq:biased_reward_mse}
\mathbb{E}\!\left[\|\tilde{g}-g_H\|^2\right]
= \frac{1}{K}\mathrm{Var}[\tilde{g}^{(1)}]
  + \|\nabla_\theta B_H(\theta)\|^2.
\end{equation}
Consequently, when $\nabla_\theta B_H(\theta) \neq 0$, averaging over more
trajectories reduces the variance term but does not remove the bias as an
estimator of $g_H$.
\end{restatable}

\begin{proof}
See \Cref{app:proof_biased_reward}.
\end{proof}

\section{Discussion}
\label{sec:discussion}

Although the dynamics $f:\mathcal{S}\times\mathcal{A}\to\mathcal{S}$ and reward $r:\mathcal{S}\times\mathcal{A}\to\mathbb{R}$ depend only on the current state and action, this does not require the underlying environment to be memoryless: any dependence on past states and actions can be absorbed into $\mathcal{S}$ by taking the state to encode a summary of the past, e.g.\ via recurrent neural networks \citep{rumelhart1986learning, elman1990finding, hochreiter1997long, henaff2016recurrent}.

\section{Acknowledgments}
Nadav Timor thanks GitHub for Startups and Lambda AI for generous financial support.
Micah Goldblum was supported by Dream Sports and the Google Research Award.
David Harel was supported by a research grant from Magnus Konow in honour of his mother Olga Konow Rappaport.

\bibliographystyle{plainnat}
\bibliography{references}

@article{williams1992simple,
  title={Simple statistical gradient-following algorithms for connectionist reinforcement learning},
  author={Williams, Ronald J},
  journal={Machine learning},
  volume={8},
  number={3},
  pages={229--256},
  year={1992},
  publisher={Springer}
}

@inproceedings{asadi2018lipschitz,
  title={Lipschitz continuity in model-based reinforcement learning},
  author={Asadi, Kavosh and Misra, Dipendra and Littman, Michael},
  booktitle={International conference on machine learning},
  pages={264--273},
  year={2018},
  organization={PMLR}
}

@article{hafner2025mastering,
  title={Mastering diverse control tasks through world models},
  author={Hafner, Danijar and Pasukonis, Jurgis and Ba, Jimmy and Lillicrap, Timothy},
  journal={Nature},
  volume={640},
  number={8059},
  pages={647--653},
  year={2025},
  publisher={Nature Publishing Group UK London}
}

@article{hafner2025training,
  title={Training agents inside of scalable world models},
  author={Hafner, Danijar and Yan, Wilson and Lillicrap, Timothy},
  journal={arXiv preprint arXiv:2509.24527},
  year={2025}
}

@article{janner2019trust,
  title={When to trust your model: Model-based policy optimization},
  author={Janner, Michael and Fu, Justin and Zhang, Marvin and Levine, Sergey},
  journal={Advances in neural information processing systems},
  volume={32},
  year={2019}
}

@incollection{sutton1990integrated,
  title={Integrated architectures for learning, planning, and reacting based on approximating dynamic programming},
  author={Sutton, Richard S},
  booktitle={Machine learning proceedings 1990},
  pages={216--224},
  year={1990},
  publisher={Elsevier}
}

@article{rumelhart1986learning,
  title={Learning representations by back-propagating errors},
  author={Rumelhart, David E and Hinton, Geoffrey E and Williams, Ronald J},
  journal={nature},
  volume={323},
  number={6088},
  pages={533--536},
  year={1986},
  publisher={Nature Publishing Group UK London}
}

@article{elman1990finding,
  title={Finding structure in time},
  author={Elman, Jeffrey L},
  journal={Cognitive science},
  volume={14},
  number={2},
  pages={179--211},
  year={1990},
  publisher={Wiley Online Library}
}

@article{hochreiter1997long,
  title={Long short-term memory},
  author={Hochreiter, Sepp and Schmidhuber, J{\"u}rgen},
  journal={Neural computation},
  volume={9},
  number={8},
  pages={1735--1780},
  year={1997},
  publisher={MIT press}
}

@inproceedings{henaff2016recurrent,
  title={Recurrent orthogonal networks and long-memory tasks},
  author={Henaff, Mikael and Szlam, Arthur and LeCun, Yann},
  booktitle={International Conference on Machine Learning},
  pages={2034--2042},
  year={2016},
  organization={PMLR}
}

@article{wang2026temporal,
  title={Temporal straightening for latent planning},
  author={Wang, Ying and Bounou, Oumayma and Zhou, Gaoyue and Balestriero, Randall and Rudner, Tim GJ and LeCun, Yann and Ren, Mengye},
  journal={arXiv preprint arXiv:2603.12231},
  year={2026}
}

@article{lecun2022path,
  title={A path towards autonomous machine intelligence version 0.9. 2, 2022-06-27},
  author={LeCun, Yann and others},
  journal={Open Review},
  volume={62},
  number={1},
  pages={1--62},
  year={2022}
}

@article{munos2003,
  title={Error bounds for approximate policy iteration},
  author={Munos, R{\'e}mi},
  journal={Proceedings of the International Conference on Machine Learning (ICML)},
  pages={560--567},
  year={2003}
}

@inproceedings{kakade2002,
  title={Approximately optimal approximate reinforcement learning},
  author={Kakade, Sham and Langford, John},
  booktitle={Proceedings of the International Conference on Machine Learning (ICML)},
  pages={267--274},
  year={2002}
}

@article{ha2018world,
  title={World models},
  author={Ha, David and Schmidhuber, J{\"u}rgen},
  journal={arXiv preprint arXiv:1803.10122},
  volume={2},
  number={3},
  pages={440},
  year={2018}
}

@inproceedings{
Hafner2020Dream,
title={Dream to Control: Learning Behaviors by Latent Imagination},
author={Danijar Hafner and Timothy Lillicrap and Jimmy Ba and Mohammad Norouzi},
booktitle={International Conference on Learning Representations},
year={2020},
url={https://openreview.net/forum?id=S1lOTC4tDS}
}

@article{schrittwieser2020mastering,
  title={Mastering atari, go, chess and shogi by planning with a learned model},
  author={Schrittwieser, Julian and Antonoglou, Ioannis and Hubert, Thomas and Simonyan, Karen and Sifre, Laurent and Schmitt, Simon and Guez, Arthur and Lockhart, Edward and Hassabis, Demis and Graepel, Thore and others},
  journal={Nature},
  volume={588},
  number={7839},
  pages={604--609},
  year={2020},
  publisher={Nature Publishing Group UK London}
}

@inproceedings{deisenroth2011pilco,
  title={PILCO: A model-based and data-efficient approach to policy search},
  author={Deisenroth, Marc and Rasmussen, Carl E},
  booktitle={Proceedings of the 28th International Conference on machine learning (ICML-11)},
  pages={465--472},
  year={2011}
}

@InProceedings{hansen2022temporal,
  title = 	 {Temporal Difference Learning for Model Predictive Control},
  author =       {Hansen, Nicklas A and Su, Hao and Wang, Xiaolong},
  booktitle = 	 {Proceedings of the 39th International Conference on Machine Learning},
  pages = 	 {8387--8406},
  year = 	 {2022},
  editor = 	 {Chaudhuri, Kamalika and Jegelka, Stefanie and Song, Le and Szepesvari, Csaba and Niu, Gang and Sabato, Sivan},
  volume = 	 {162},
  series = 	 {Proceedings of Machine Learning Research},
  month = 	 {17--23 Jul},
  publisher =    {PMLR},
  url = 	 {https://proceedings.mlr.press/v162/hansen22a.html}
}

@article{kearns2002near,
  title={Near-optimal reinforcement learning in polynomial time},
  author={Kearns, Michael and Singh, Satinder},
  journal={Machine learning},
  volume={49},
  number={2},
  pages={209--232},
  year={2002},
  publisher={Springer}
}

@inproceedings{farahmand2017value,
  title={Value-aware loss function for model-based reinforcement learning},
  author={Farahmand, Amir-massoud and Barreto, Andre and Nikovski, Daniel},
  booktitle={Artificial Intelligence and Statistics},
  pages={1486--1494},
  year={2017},
  organization={PMLR}
}

@article{asadi2018equivalence,
  title={Equivalence between wasserstein and value-aware loss for model-based reinforcement learning},
  author={Asadi, Kavosh and Cater, Evan and Misra, Dipendra and Littman, Michael L},
  journal={arXiv preprint arXiv:1806.01265},
  year={2018}
}

@inproceedings{talvitie2018learning,
  title={Learning the reward function for a misspecified model},
  author={Talvitie, Erik},
  booktitle={International Conference on Machine Learning},
  pages={4838--4847},
  year={2018},
  organization={PMLR}
}

@inproceedings{
lobel2024optimal,
title={An Optimal Tightness Bound for the Simulation Lemma},
author={Sam Lobel and Ronald Parr},
booktitle={Reinforcement Learning Conference},
year={2024},
url={https://openreview.net/forum?id=RcoIAfiM5g}
}

@article{kaplan2020scaling,
  title={Scaling laws for neural language models},
  author={Kaplan, Jared and McCandlish, Sam and Henighan, Tom and Brown, Tom B and Chess, Benjamin and Child, Rewon and Gray, Scott and Radford, Alec and Wu, Jeffrey and Amodei, Dario},
  journal={arXiv preprint arXiv:2001.08361},
  year={2020}
}

@article{hoffmann2022training,
  title={Training compute-optimal large language models},
  author={Hoffmann, Jordan and Borgeaud, Sebastian and Mensch, Arthur and Buchatskaya, Elena and Cai, Trevor and Rutherford, Eliza and Casas, DDL and Hendricks, Lisa Anne and Welbl, Johannes and Clark, Aidan and others},
  journal={arXiv preprint arXiv:2203.15556},
  volume={10},
  year={2022}
}

@article{hilton2023scaling,
  title={Scaling laws for single-agent reinforcement learning},
  author={Hilton, Jacob and Tang, Jie and Schulman, John},
  journal={arXiv preprint arXiv:2301.13442},
  year={2023}
}

@inproceedings{pearce2025scaling,
  title={Scaling Laws for Pre-training Agents and World Models},
  author={Pearce, Tim and Rashid, Tabish and Bignell, David and Georgescu, Raluca and Devlin, Sam and Hofmann, Katja},
  booktitle={International Conference on Machine Learning},
  pages={48542--48562},
  year={2025},
  organization={PMLR}
}

@inproceedings{jin2020reward,
  title={Reward-free exploration for reinforcement learning},
  author={Jin, Chi and Krishnamurthy, Akshay and Simchowitz, Max and Yu, Tiancheng},
  booktitle={International Conference on Machine Learning},
  pages={4870--4879},
  year={2020},
  organization={PMLR}
}

@article{bellinger2020active,
  title={Active Measure Reinforcement Learning for Observation Cost Minimization},
  author={Bellinger, Colin and Coles, Rory and Crowley, Mark and Tamblyn, Isaac},
  journal={arXiv preprint arXiv:2005.12697},
  year={2020}
}

@article{huang2022efficient,
  title={An Efficient Simulation-Based Policy Improvement with Optimal Computing Budget Allocation Based on Accumulated Samples},
  author={Huang, Xilang and Choi, Seon Han},
  journal={Electronics},
  volume={11},
  number={7},
  pages={1141},
  year={2022},
  publisher={MDPI}
}

@article{sutton1999policy,
  title={Policy gradient methods for reinforcement learning with function approximation},
  author={Sutton, Richard S and McAllester, David and Singh, Satinder and Mansour, Yishay},
  journal={Advances in neural information processing systems},
  volume={12},
  year={1999}
}

@article{greensmith2004variance,
  title={Variance reduction techniques for gradient estimates in reinforcement learning},
  author={Greensmith, Evan and Bartlett, Peter L and Baxter, Jonathan},
  journal={Journal of Machine Learning Research},
  volume={5},
  number={Nov},
  pages={1471--1530},
  year={2004}
}

@inproceedings{zhang2021robust,
  title={Robust policy gradient against strong data corruption},
  author={Zhang, Xuezhou and Chen, Yiding and Zhu, Xiaojin and Sun, Wen},
  booktitle={International Conference on Machine Learning},
  pages={12391--12401},
  year={2021},
  organization={PMLR}
}

@article{cai2025reinforcement,
  title={Reinforcement learning with verifiable yet noisy rewards under imperfect verifiers},
  author={Cai, Xin-Qiang and Wang, Wei and Liu, Feng and Liu, Tongliang and Niu, Gang and Sugiyama, Masashi},
  journal={arXiv preprint arXiv:2510.00915},
  year={2025}
}

@article{han2026non,
  title={Non-Uniform Noise-to-Signal Ratio in the REINFORCE Policy-Gradient Estimator},
  author={Han, Haoyu and Yang, Heng},
  journal={arXiv preprint arXiv:2602.01460},
  year={2026}
}

@article{christiano2017deep,
  title={Deep reinforcement learning from human preferences},
  author={Christiano, Paul F and Leike, Jan and Brown, Tom and Martic, Miljan and Legg, Shane and Amodei, Dario},
  journal={Advances in neural information processing systems},
  volume={30},
  year={2017}
}

@article{stiennon2020learning,
  title={Learning to summarize with human feedback},
  author={Stiennon, Nisan and Ouyang, Long and Wu, Jeffrey and Ziegler, Daniel and Lowe, Ryan and Voss, Chelsea and Radford, Alec and Amodei, Dario and Christiano, Paul F},
  journal={Advances in neural information processing systems},
  volume={33},
  pages={3008--3021},
  year={2020}
}

@inproceedings{gao2023scaling,
  title={Scaling laws for reward model overoptimization},
  author={Gao, Leo and Schulman, John and Hilton, Jacob},
  booktitle={International Conference on Machine Learning},
  pages={10835--10866},
  year={2023},
  organization={PMLR}
}

@inproceedings{assran2023self,
  title={Self-supervised learning from images with a joint-embedding predictive architecture},
  author={Assran, Mahmoud and Duval, Quentin and Misra, Ishan and Bojanowski, Piotr and Vincent, Pascal and Rabbat, Michael and LeCun, Yann and Ballas, Nicolas},
  booktitle={Proceedings of the IEEE/CVF conference on computer vision and pattern recognition},
  pages={15619--15629},
  year={2023}
}

@inproceedings{miyato2018spectral,
  title={Spectral Normalization for Generative Adversarial Networks},
  author={Miyato, Takeru and Kataoka, Toshiki and Koyama, Masanori and Yoshida, Yuichi},
  booktitle={International Conference on Learning Representations},
  year={2018}
}

@article{gouk2021regularisation,
  title={Regularisation of neural networks by enforcing lipschitz continuity},
  author={Gouk, Henry and Frank, Eibe and Pfahringer, Bernhard and Cree, Michael J},
  journal={Machine Learning},
  volume={110},
  number={2},
  pages={393--416},
  year={2021},
  publisher={Springer}
}


\appendix

\section{Additional related work}
\label{app:additional_related_work}

This appendix collects related works that inform the setting of \cref{sec:related_work} but are not load-bearing for the novelty claims of \cref{lem:decomposed,thm:optimal_split,thm:noise,cor:fidelity,prop:biased_reward,cor:lipschitz_tightens_bound,prop:temporal_straightening}.

\paragraph{Latent-state world models.}
Beyond the works cited in \cref{sec:related_work}, several lines of work shape the modern training-in-imagination paradigm. \citet{deisenroth2011pilco} fit a Gaussian-process dynamics model and propagate uncertainty through imagined trajectories under a known reward. \citet{ha2018world} learn a latent dynamics model from pixels and train a policy primarily inside this learned model. Dreamer \citep{Hafner2020Dream} learns a latent-space dynamics model together with a learned reward predictor and optimizes the policy entirely on imagined latent rollouts. \citet{schrittwieser2020mastering} learn a value-equivalent latent model whose dynamics, reward, and value predictions are trained jointly to support planning, and \citet{hansen2022temporal} combine a learned latent dynamics and reward model with short-horizon planning and a learned terminal value.

\paragraph{Simulation-style bounds and value-aware model learning.}
The simulation-lemma family extends in several directions. \citet{kakade2002} express the difference in expected discounted return between two policies as an expectation of single-step advantages under one of them, which underpins conservative policy iteration. \citet{munos2003} gives $L_p$-style error-propagation bounds for approximate policy iteration, and \citet{lobel2024optimal} recently establish optimal tightness for the simulation lemma. A parallel line of work asks whether the model loss should target value prediction rather than raw transition accuracy: value-aware model learning \citep{farahmand2017value} replaces the next-state likelihood objective with a loss measuring the worst-case discrepancy between the true and learned dynamics on expected values over a class of value functions, \citet{asadi2018equivalence} show that, restricted to a 1-Lipschitz value-function class, this loss is equivalent to the Wasserstein distance between the dynamics, and \citet{talvitie2018learning} addresses how to learn the reward function on states drawn from a misspecified dynamics model.

\paragraph{Active observation and simulation budget allocation.}
Active-measure reinforcement learning chooses when to pay for an observation under explicit observation costs \citep{bellinger2020active}. In the simulation-optimization literature, optimal computing budget allocation divides a fixed simulation budget across candidate policies to maximize the probability of selecting the best one \citep{huang2022efficient}. Power-law fits analogous to those of \citet{kaplan2020scaling} and \citet{hoffmann2022training} have since been reported for single-agent reinforcement learning and for pre-training of agents and world models \citep{hilton2023scaling, pearce2025scaling}.

\paragraph{Policy-gradient theory and reward modeling.}
Beyond \citet{williams1992simple}, the policy-gradient theorem of \citet{sutton1999policy} and the variance-reduction analysis of \citet{greensmith2004variance} provide the standard analytic toolkit invoked by \cref{thm:noise}. \citet{han2026non} characterize how the \reinforce noise-to-signal ratio varies non-uniformly across the parameter landscape. Reinforcement learning from human feedback trains reward models from preferences \citep{christiano2017deep, stiennon2020learning}, the empirical setting in which \citet{gao2023scaling} document Goodhart-style overoptimization.

\paragraph{Representation learning and Lipschitz regularization.}
\citet{assran2023self} provide a concrete image instantiation of joint-embedding predictive architectures, and \citet{gouk2021regularisation} propose projection-based mechanisms for enforcing Lipschitz continuity during training, with upper bounds applicable to multiple $p$-norms beyond the spectral norm.

\section{Full proofs}
\label{app:proofs}

\subsection{Proof of Lemma~\ref{lem:decomposed}}
\label{app:proof_decomposed}

\decomposedlemma*

\begin{proof}
Fix an $L_\pi$-Lipschitz policy $\pi$ and a shared initial state
$s_0 = \hat{s}_0$.
Let $\{s_t\}$ denote the state trajectory generated by executing $\pi$ under the true dynamics $f$, and $\{\hat{s}_t\}$ the trajectory under $\hat{f}$, both starting from $s_0$.
The actions differ because the policy is evaluated at different states:
\begin{align*}
a_t = \pi(s_t), \qquad \hat{a}_t = \pi(\hat{s}_t).
\end{align*}

At each step $t$:
\begin{align}
r(s_t, a_t) - \hat{r}(\hat{s}_t, \hat{a}_t)
&= \bigl[r(s_t, a_t) - r(\hat{s}_t, \hat{a}_t)\bigr]
   + \bigl[r(\hat{s}_t, \hat{a}_t) - \hat{r}(\hat{s}_t, \hat{a}_t)\bigr]. \label{eq:per_step_decomp}
\end{align}

Let $\Lcomp := L_f(1 + L_\pi)$.
By induction:
\begin{equation*}
\|s_t - \hat{s}_t\| \leq \edyn \sum_{k=0}^{t-1} \Lcomp^k
\end{equation*}

\emph{Base case} ($t=0$): $s_0 = \hat{s}_0$, so $\|s_0 - \hat{s}_0\| = 0$.

\emph{Inductive step}: Assume $\|s_t - \hat{s}_t\| \leq \edyn \sum_{k=0}^{t-1} \Lcomp^k$.
\begin{align}
\|s_{t+1} - \hat{s}_{t+1}\|
&= \|f(s_t, a_t) - \hat{f}(\hat{s}_t, \hat{a}_t)\| \notag \\
&\leq \|f(s_t, a_t) - f(\hat{s}_t, \hat{a}_t)\| + \|f(\hat{s}_t, \hat{a}_t) - \hat{f}(\hat{s}_t, \hat{a}_t)\| \notag \\
&\leq L_f\bigl(\|s_t - \hat{s}_t\| + \|a_t - \hat{a}_t\|\bigr) + \edyn \notag \\
&= L_f\bigl(\|s_t - \hat{s}_t\| + \|\pi(s_t) - \pi(\hat{s}_t)\|\bigr) + \edyn \notag \\
&\leq L_f\bigl(\|s_t - \hat{s}_t\| + L_\pi\|s_t - \hat{s}_t\|\bigr) + \edyn \notag \\
&\leq L_f(1 + L_\pi)\|s_t - \hat{s}_t\| + \edyn \notag \\
&= \Lcomp\|s_t - \hat{s}_t\| + \edyn \notag \\
&\leq \Lcomp\left(\edyn \sum_{k=0}^{t-1} \Lcomp^k\right) + \edyn \notag \\
&= \edyn \sum_{k=0}^{t-1} \Lcomp^{k+1} + \edyn \notag \\
&= \edyn \sum_{k=0}^{t-1} \Lcomp^{k+1} + \edyn \Lcomp^0 \notag \\
&= \edyn \sum_{j=1}^{t} \Lcomp^j + \edyn \Lcomp^0 \notag \\
&= \edyn \sum_{k=0}^{t} \Lcomp^k. \notag
\end{align}

\emph{Error from using the learned reward function:}
$|r(\hat{s}_t, \hat{a}_t) - \hat{r}(\hat{s}_t, \hat{a}_t)| \leq \erew$, by the definition of $\erew$.

\emph{Error from evaluating the true reward at different state-action pairs:}
\begin{align}
|r(s_t, a_t) - r(\hat{s}_t, \hat{a}_t)|
&\leq L_r\bigl(\|s_t - \hat{s}_t\| + \|a_t - \hat{a}_t\|\bigr) \notag \\
&= L_r\bigl(\|s_t - \hat{s}_t\| + \|\pi(s_t) - \pi(\hat{s}_t)\|\bigr) \notag \\
&\leq L_r\bigl(\|s_t - \hat{s}_t\| + L_\pi\|s_t - \hat{s}_t\|\bigr) \notag \\
&= L_r(1 + L_\pi)\|s_t - \hat{s}_t\| \notag \\
&\leq L_r(1+L_\pi) \edyn \sum_{k=0}^{t-1} \Lcomp^k. \notag
\end{align}

Using the definition of discounted return and \cref{eq:per_step_decomp},
\begin{align}
|J(\pi, \mathcal{M}) - J(\pi, \hat{\mathcal{M}})|
&= \left|\sum_{t=0}^{\infty} \gamma^t \bigl(r(s_t, a_t) - \hat{r}(\hat{s}_t, \hat{a}_t)\bigr)\right| \notag \\
&\leq \sum_{t=0}^{\infty} \gamma^t \left|r(s_t, a_t) - \hat{r}(\hat{s}_t, \hat{a}_t)\right| \notag \\
&= \sum_{t=0}^{\infty} \gamma^t \left|\bigl(r(s_t, a_t) - r(\hat{s}_t, \hat{a}_t)\bigr)
    + \bigl(r(\hat{s}_t, \hat{a}_t) - \hat{r}(\hat{s}_t, \hat{a}_t)\bigr)\right| \notag \\
&\leq \sum_{t=0}^{\infty} \gamma^t \left|r(s_t, a_t) - r(\hat{s}_t, \hat{a}_t)\right|
    + \sum_{t=0}^{\infty} \gamma^t \left|r(\hat{s}_t, \hat{a}_t) - \hat{r}(\hat{s}_t, \hat{a}_t)\right| \notag \\
&\leq L_r(1+L_\pi) \edyn \sum_{t=0}^{\infty} \gamma^t \left(\sum_{k=0}^{t-1} \Lcomp^k\right)
    + \erew \sum_{t=0}^{\infty} \gamma^t \notag \\
&= L_r(1+L_\pi) \edyn \sum_{t=0}^{\infty}\sum_{k=0}^{t-1} \gamma^t \Lcomp^k
    + \frac{1}{1 - \gamma}\,\erew \notag \\
&= L_r(1+L_\pi) \edyn \sum_{k=0}^{\infty} \sum_{\substack{t=0 \\ k \le t-1}}^{\infty} \gamma^t \Lcomp^k
    + \frac{1}{1 - \gamma}\,\erew \notag \\
&= L_r(1+L_\pi) \edyn \sum_{k=0}^{\infty} \sum_{t=k+1}^{\infty} \gamma^t \Lcomp^k
    + \frac{1}{1 - \gamma}\,\erew \notag \\
&= L_r(1+L_\pi) \edyn \sum_{k=0}^{\infty} \sum_{j=0}^{\infty} \gamma^{k+1+j}\Lcomp^k
    + \frac{1}{1 - \gamma}\,\erew \notag \\
&= L_r(1+L_\pi) \edyn \sum_{k=0}^{\infty} \left(\gamma^{k+1}\sum_{j=0}^{\infty} \gamma^j\right)\Lcomp^k
    + \frac{1}{1 - \gamma}\,\erew \notag \\
&= \frac{\gamma L_r(1+L_\pi)}{1-\gamma}\,\edyn \sum_{k=0}^{\infty} (\gamma \Lcomp)^k
    + \frac{1}{1 - \gamma}\,\erew \notag \\
\intertext{Because $\gamma\Lcomp < 1$, the geometric series converges.}
&= \frac{\gamma L_r(1+L_\pi)}{1-\gamma}\,\edyn \left(\frac{1}{1-\gamma\Lcomp}\right)
    + \frac{1}{1 - \gamma}\,\erew \notag \\
&= \frac{1}{1 - \gamma}\,\erew + \splitcoef\,\edyn \notag
\end{align}
\end{proof}

\subsection[Proof of the temporal straightening proposition]{Proof of \Cref{prop:temporal_straightening}}
\label{app:proof_temporal_straightening}

\straighteningthm*

\begin{proof}
Fix $t$ such that $v(z_t)\neq 0$ and $v(z_{t+1})\neq 0$, and set
\begin{equation*}
a:=v(z_t),
\qquad
b:=v(z_{t+1}).
\end{equation*}
Along the rollout, $a=z_{t+1}-z_t$ and $b=z_{t+2}-z_{t+1}$.
Using the $\varepsilon$-Lipschitz condition for $v$ gives
\begin{align*}
\|b-a\|
&= \|v(z_{t+1})-v(z_t)\| \\
&\leq \varepsilon\|z_{t+1}-z_t\| \\
&= \varepsilon\|a\|.
\end{align*}

By the triangle inequality,
\begin{align*}
\|b\|
&= \|a+(b-a)\| \\
&\geq \|a\|-\|b-a\| \\
&\geq (1-\varepsilon)\|a\|.
\end{align*}
Define
\begin{equation*}
C := \frac{a^\top b}{\|a\|\|b\|}.
\end{equation*}
Then
\begin{align*}
\|b-a\|^2
&= \|a\|^2+\|b\|^2-2a^\top b \\
&= \|a\|^2+\|b\|^2-2C\|a\|\|b\| \\
&= (\|a\|-\|b\|)^2 + 2\|a\|\|b\|(1-C).
\end{align*}
Rearranging and using $(\|a\|-\|b\|)^2\geq 0$ yields
\begin{align*}
1-C
&= \frac{\|b-a\|^2-(\|a\|-\|b\|)^2}{2\|a\|\|b\|} \\
&\leq \frac{\|b-a\|^2}{2\|a\|\|b\|}.
\end{align*}
Substituting $\|b-a\|\leq \varepsilon\|a\|$ and $\|b\|\geq (1-\varepsilon)\|a\|$ gives
\begin{align*}
1-C
&\leq \frac{\varepsilon^2\|a\|^2}{2\|a\|(1-\varepsilon)\|a\|} \\
&= \frac{\varepsilon^2}{2(1-\varepsilon)}.
\end{align*}
Since $1-C=\mathcal{L}_{\mathrm{curv}}(t)$ by definition, this proves the curvature-loss bound.

The function $\varepsilon \mapsto \varepsilon^2/(2(1-\varepsilon))$ is increasing on $(0,1)$ because
\begin{equation*}
\frac{d}{d\varepsilon}\frac{\varepsilon^2}{2(1-\varepsilon)}
= \frac{\varepsilon(2-\varepsilon)}{2(1-\varepsilon)^2}>0.
\end{equation*}
Therefore lowering the Lipschitz constant of $v$ lowers the upper bound whenever the Lipschitz constant remains in $(0,1)$.
\end{proof}

\subsection{Proof of \Cref{thm:optimal_split}}
\label{app:proof_optimal_split}

\optimalsplitthm*

\begin{proof}
For sample counts $(\Ndyn, \Nrew)$, \Cref{eq:error_scaling} gives
\begin{equation*}
\erew(\Nrew) = A_r \Nrew^{-\beta},
\qquad
\edyn(\Ndyn) = A_d \Ndyn^{-\alpha}.
\end{equation*}
Substituting these identities into the upper bound from
\Cref{eq:decomposed_bound} gives
\begin{equation*}
\frac{A_r}{1-\gamma}\,\Nrew^{-\beta}
+ \splitcoef\,A_d\,\Ndyn^{-\alpha}.
\end{equation*}
Define the sample-count-independent coefficients by
\begin{equation*}
C_r := \frac{A_r}{1-\gamma},
\qquad
C_d := \frac{\gamma L_r(1+L_\pi) A_d}{(1-\gamma)(1-\gamma L_f(1+L_\pi))},
\end{equation*}
so the upper bound after substitution becomes
\begin{equation*}
\mathcal{L}_{\mathrm{bound}}(\Ndyn, \Nrew)
:= C_r \Nrew^{-\beta} + C_d \Ndyn^{-\alpha}.
\end{equation*}
For the fixed budget $B$, the feasible sample counts are exactly the pairs
$(\Ndyn, \Nrew)$ satisfying $\cdyn \Ndyn + \crew \Nrew = B$. Therefore
minimizing the upper bound from \Cref{eq:decomposed_bound} over all feasible
sample counts is equivalent to minimizing
$\mathcal{L}_{\mathrm{bound}}(\Ndyn, \Nrew)$ subject to the same budget
constraint.

Because $\alpha,\beta > 0$, the objective decreases as either sample count
increases, so the minimizer satisfies $\Ndyn^* > 0$ and $\Nrew^* > 0$. Define
the Lagrangian
\begin{equation*}
\Lambda(\Ndyn, \Nrew, \lambda)
:= C_d \Ndyn^{-\alpha} + C_r \Nrew^{-\beta}
+ \lambda(\cdyn \Ndyn + \crew \Nrew - B).
\end{equation*}
Because the minimizing counts satisfy $\Ndyn^* > 0$ and $\Nrew^* > 0$,
setting the partial derivatives of $\Lambda$ with respect to $\Ndyn$ and
$\Nrew$ equal to zero gives
\begin{align*}
\alpha\, C_d\, \Ndyn^{-\alpha - 1} &= \lambda\, \cdyn, \\
\beta\, C_r\, \Nrew^{-\beta - 1} &= \lambda\, \crew.
\end{align*}

Multiply the equation for $\Ndyn$ by $\Ndyn$ and the equation for $\Nrew$ by
$\Nrew$:
\begin{align*}
\alpha\, C_d\, \Ndyn^{-\alpha} &= \lambda\, \cdyn \Ndyn, \\
\beta\, C_r\, \Nrew^{-\beta} &= \lambda\, \crew \Nrew.
\end{align*}
Now divide the equation for $\Ndyn$ by the equation for $\Nrew$:
\begin{align*}
\frac{\alpha\, C_d\, \Ndyn^{-\alpha}}{\beta\, C_r\, \Nrew^{-\beta}}
&= \frac{\cdyn \Ndyn}{\crew \Nrew}.
\end{align*}
Rearranging this identity gives
\begin{equation}
\label{eq:proof_sample_ratio}
\frac{\Ndyn}{\Nrew}
= \frac{\alpha}{\beta} \cdot \frac{\crew}{\cdyn}
   \cdot \frac{C_d \Ndyn^{-\alpha}}{C_r \Nrew^{-\beta}}.
\end{equation}
Evaluate \Cref{eq:proof_sample_ratio} at the minimizing counts
$(\Ndyn^*, \Nrew^*)$.
Using the definition of $C_d$ together with
\cref{eq:error_scaling}, we obtain
\begin{align*}
C_d (\Ndyn^*)^{-\alpha}
&= \splitcoef\,A_d (\Ndyn^*)^{-\alpha} \\
&= \splitcoef\,\edyn^*, \\
C_r (\Nrew^*)^{-\beta}
&= \frac{A_r}{1-\gamma} (\Nrew^*)^{-\beta} \\
&= \frac{1}{1-\gamma}\,\erew^*.
\end{align*}
Substituting these two identities yields
\begin{align*}
\frac{\Ndyn^*}{\Nrew^*}
&= \frac{\alpha}{\beta} \cdot (1-\gamma)\splitcoef
   \cdot \frac{\crew}{\cdyn} \cdot \frac{\edyn^*}{\erew^*} \\
&= \frac{\alpha}{\beta} \cdot \frac{\gamma L_r(1+L_\pi)}{1-\gamma L_f(1+L_\pi)}
   \cdot \frac{\crew}{\cdyn} \cdot \frac{\edyn^*}{\erew^*}.
\end{align*}

Because $x \mapsto x^{-\alpha}$ and $x \mapsto x^{-\beta}$ are convex on
$(0,\infty)$ for $\alpha,\beta > 0$, the objective is convex on the feasible
set. Therefore the feasible point satisfying the derivative equations for
$\Ndyn$ and $\Nrew$ is the unique minimizer.
\end{proof}

\subsection{Experiment details: empirical calibration of the decomposed bound}
\label{app:lemma_calibration_details}

This appendix collects the protocol and per-configuration construction for the calibration experiment in \Cref{sec:lemma_calibration}, whose main-text result is summarized in \Cref{fig:lemma_calibration}.

\begin{figure}[t]
  \centering
  \includegraphics[width=0.85\linewidth]{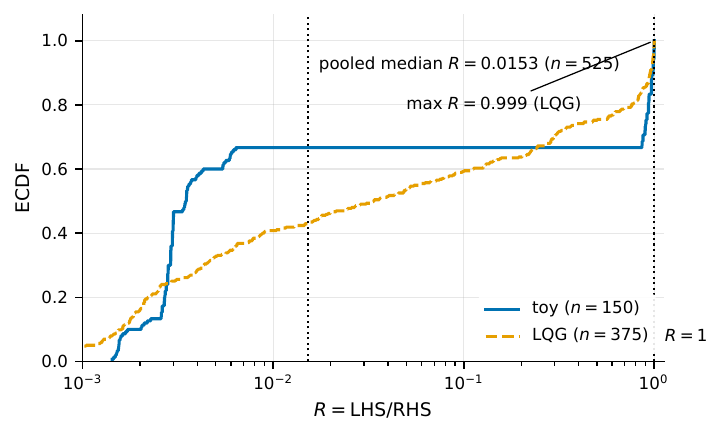}
  \caption{The bound of \cref{eq:decomposed_bound} holds across all $n=525$ test configurations on both benchmarks.
  The empirical CDF of $R := \mathrm{LHS}/\mathrm{RHS}$ stays at or below $R = 1$ on every configuration, supporting the use of \Cref{lem:decomposed} as a hypothesis of \Cref{thm:optimal_split}; full per-benchmark statistics are in this appendix.}
  \label{fig:lemma_calibration}
\end{figure}

Each configuration is a tuple $(f, r, \pi, \hat f, \hat r)$, where $(f, r)$ is an \mdp instance, $\pi$ is an $L_\pi$-Lipschitz evaluation policy, and $(\hat f, \hat r)$ is the perturbed model pair. The realized per-step errors $\edyn, \erew$ are computed from the $(f, \hat f)$ and $(r, \hat r)$ pair using the same definitions that the main-text formulas use.

The synthetic benchmark contains $n = 150$ configurations in which $f$ and $r$ are globally Lipschitz; on this benchmark $L_f, L_r, L_\pi$ in $\mathrm{RHS}$ of \cref{eq:decomposed_bound} are computed analytically as the global Lipschitz constants of the configuration's maps. The LQG benchmark contains $n = 375$ configurations in which the quadratic reward is Lipschitz only on a bounded operating domain (a fixed compact subset of state-action space within which the configuration's rollouts are confined); on this benchmark $L_r$ is the Lipschitz constant of $r$ restricted to that domain.

The empirical CDFs of $R$ in \Cref{fig:lemma_calibration} have qualitatively different shapes on the two benchmarks. The LQG ECDF climbs steadily across $R \in [10^{-3}, 1]$. The synthetic ECDF rises sharply through its low-$R$ mass, stays flat at value $\approx 0.66$ across $R \in [10^{-2},\, 5\times 10^{-1}]$, and then climbs toward $R = 1$ on the remaining $\approx 25\%$ of configurations. The pooled median $R = 0.015$ is therefore dominated by the LQG arm. The pool maximum is $R = 0.9995$ on a synthetic configuration, and the LQG arm peaks at $R = 0.999$.

Two annotations on \Cref{fig:lemma_calibration} should be read as follows: $\texttt{max~R = 0.999~(LQG)}$ is the LQG-arm maximum (the pool maximum is the slightly larger synthetic value $R = 0.9995$), and the legend label \texttt{toy} denotes the synthetic benchmark.

We report only the per-benchmark medians, the pooled median, and the per-benchmark maxima; no bootstrap confidence intervals are used because the relevant claim is the universal $R \le 1$, which is verified directly on every configuration.


\subsection{Experiment details: empirical evaluation of the optimal sample-ratio formula}
\label{app:allocation_evaluation_details}

This appendix collects the per-configuration construction, definitions, statistical methodology, and numerical results for the allocation-evaluation experiment in \Cref{sec:allocation_evaluation}, whose main-text findings are summarized in \Cref{fig:allocation_form,fig:lipschitz_slack}.

The configurations used in \Cref{fig:allocation_form,fig:lipschitz_slack} are partitioned into four groups, each generated by a Cartesian product over a small set of axes; per-configuration values are recorded in the committed CSVs and we list the grids verbatim below. The interpretation of the axes \texttt{sigma\_ratio}, \texttt{cost\_ratio}, $\lambda$, and \texttt{theta\_f0} is deferred to the open item at the end of this appendix.

\paragraph{Linear-value configurations ($n = 30$, panel~(a) of \Cref{fig:allocation_form}).}
The Cartesian product is $(L_f, \lambda) \in \{(0.5, 0.5),\; (2.0, 2.0)\}$ (with $L_f = \lambda$ enforced) crossed with $\texttt{sigma\_ratio} \in \{0.1, 0.3, 1.0, 3.0, 10.0\}$ and $\texttt{cost\_ratio} \in \{0.1, 1.0, 10.0\}$, for $2 \times 5 \times 3 = 30$ configurations. The reward Lipschitz constant is fixed at $L_r = 1$.

\paragraph{$\tanh$-value and $\sin$-value configurations ($n = 9$ each, panel~(a) of \Cref{fig:allocation_form}).}
For each of these two value-function families, the Cartesian product is $\texttt{sigma\_ratio} \in \{0.3, 1.0, 3.0\}$ crossed with $\texttt{cost\_ratio} \in \{0.1, 1.0, 10.0\}$, with $L_f = L_r = \lambda = 1$ fixed.

\paragraph{Quadratic-value $\sup$-norm-control configurations ($n = 9$, panel~(b) of \Cref{fig:allocation_form}).}
The Cartesian product is the same $\texttt{sigma\_ratio} \times \texttt{cost\_ratio}$ grid of size~$9$, with $\lambda = 1$ and $\texttt{theta\_f0} = 0.5$ fixed. Each configuration is evaluated twice: once with $L_f^{\mathrm{local}} = 1$ (the realized-sensitivity calculation, plotted as filled circles in panel~(b)) and once with $L_f^{\sup} = 2$ (the $\sup$-norm control, plotted as open squares).

\paragraph{LQG configurations ($n = 30$, \Cref{fig:lipschitz_slack}).}
Configurations are indexed by $\texttt{seed} \in \{0, 1, \ldots, 29\}$ and parameterized so that, at $\gamma = 0.8$, the contraction quantity $\gamma L_f (1 + L_\pi)$ lies in $[0.224, 0.228]$ (margin $1 - \gamma L_f (1 + L_\pi) \in [0.772, 0.776]$). The realized constants per seed land in $L_f \in [0.280, 0.285]$, $L_\pi \in [1.6\times 10^{-4}, 1.0\times 10^{-3}]$, $L_r \in [0.82, 1.49]$.

For each seed we fit per-configuration power-law scaling laws $\edyn(N) = A_d \cdot N^{-\alpha_d}$ and $\erew(N) = A_r \cdot N^{-\alpha_r}$, with $A_d \in [1.0\times 10^{-3}, 1.7\times 10^{-3}]$, $A_r \in [2.2\times 10^{-3}, 4.0\times 10^{-3}]$, $\alpha_d \in [0.97, 1.05]$, $\alpha_r \in [0.99, 1.08]$, $R^2_d \in [0.995, 0.9997]$, and $R^2_r \in [0.996, 0.9998]$ (recorded per-configuration in the \texttt{A\_d, A\_r, alpha\_d, alpha\_r, R2\_d, R2\_r} columns of \texttt{lqg\_per\_instance.csv}). These per-seed exponents are distinct from the global $\alpha, \beta$ fit in \Cref{sec:error_scaling}.

The \texttt{boundary\_excluded} column reports that $0/30$ configurations were filtered out as boundary cases (the filter would exclude any configuration whose realized contraction $\gamma L_f(1+L_\pi)$ approached $1$, which would invalidate the hypothesis of \Cref{lem:decomposed}; none did).

We define the quantities used in this experiment. Let $V^\pi(s; f, r)$ denote the discounted return of a fixed evaluation policy $\pi$ when the dynamics map is $f$ and the reward map is $r$. For each test configuration, let $\Delta f := \hat f - f$ and $\Delta r := \hat r - r$ denote the realized model perturbations, and fix a step size $h>0$. Define the realized value sensitivities by the one-sided finite-difference quotients
\begin{align*}
  S_f(s,a) &\;:=\; \frac{|V^\pi(s; f + h\,\Delta f, r) - V^\pi(s; f, r)|}{h\,\|\Delta f\|}, \\
  S_r(s,a) &\;:=\; \frac{|V^\pi(s; f, r + h\,\Delta r) - V^\pi(s; f, r)|}{h\,|\Delta r|},
\end{align*}
where $\|\cdot\|$ is the Euclidean norm; as $h \to 0$, $S_f$ and $S_r$ approach the directional derivatives $\partial_f V^\pi$ and $\partial_r V^\pi$, which provide the underlying intuition for these realized sensitivities. The global Lipschitz constant of the dynamics map in the spectral norm is any
\begin{equation*}
  L_f \;\geq\; \sup_{(s,a)\neq(s',a')}\; \frac{\|f(s,a) - f(s',a')\|}{\|s-s'\| + \|a-a'\|},
\end{equation*}
and $L_r$ is defined analogously, on a bounded operating domain when the reward is only locally Lipschitz. Let $K_{\mathrm{Lip}} := \splitcoef$ denote the dynamics coefficient appearing in \cref{eq:decomposed_bound}, instantiated with the analytical global Lipschitz constants $L_f, L_r$, and let $K'$ denote the same coefficient with $L_f$ in the contraction factor $1 - \gamma L_f(1+L_\pi)$ replaced by the realized value-function sensitivity $S_f$. The log-ratio residual $\ell$ is defined in \Cref{sec:allocation_evaluation}.

\begin{figure}[t]
  \centering
  \includegraphics[width=0.65\linewidth]{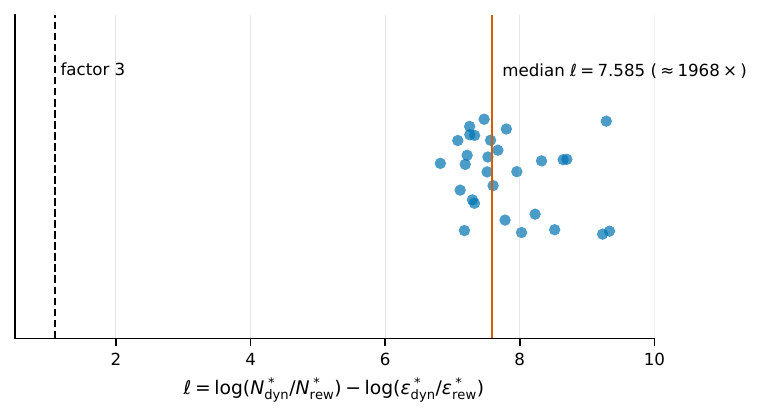}
  \caption{The multiplier in \cref{eq:optimal_ratio} instantiated with the global Lipschitz constants of \Cref{lem:decomposed} overshoots realized ratios by about three orders of magnitude on LQG. Each of $n = 30$ LQG configurations at $\gamma = 0.8$ lies above the dashed horizontal line $\ell = \log 3$ (predicted and realized ratios within a factor of $3$); the solid horizontal line marks the median $\ell = 7.585$.}
  \label{fig:lipschitz_slack}
\end{figure}

We summarize each group of configurations by the across-configuration median of $|\ell|$ (or $\ell$ when overshoot direction is informative); \Cref{fig:lipschitz_slack} jitters the vertical position of each LQG marker for visibility. To check that overshoot in the global-Lipschitz multiplier comparison is not consistent with chance, we run a one-sided sign test against the null that overprediction and underprediction are equally likely; with $30/30$ configurations overshooting, the test gives $p = 0.5^{30} \approx 9.31 \times 10^{-10}$.

The headline per-group medians are reported in \Cref{sec:allocation_evaluation}; here we record the additional numbers underlying the $\sup$-norm-control comparison. On the $9$ quadratic-value configurations, the median residual rises from $|\ell_{\mathrm{realized}}| = 0.074$ to $|\ell_{\sup}| = 0.684$ when realized sensitivities are replaced by their $\sup$-norm overestimates, a factor of $\approx 9.25\times$. The maximum $|\ell|$ on the linear-value configurations is also $0$, complementing the median value of $0$ already reported in main text.

A separate statistic isolates the dynamics-coefficient slack itself. On the same $n = 30$ LQG configurations (committed per-configuration as the \texttt{K\_prime} column of \texttt{lqg\_per\_instance.csv}), the multiplier ratio $K_{\mathrm{Lip}} / K'$ has median $2.37$ and maximum $3.23$ (minimum $1.71$). This is a different quantity from the residual $\ell$: the factor $\approx 1968$ figure is $\exp(7.585)$ on the residual itself, while $K_{\mathrm{Lip}} / K'$ isolates the contribution of replacing the global $L_f$ with the realized $S_f$ in the dynamics-side coefficient alone.

Worst-case bounds in approximate dynamic programming, such as the simulation-lemma bounds on policy-evaluation and policy-improvement error stated in terms of model error, are classically conservative when instantiated with global Lipschitz or sup-norm constants \citep{munos2003,kakade2002}. The pattern observed here matches that classical observation: the bound in \Cref{lem:decomposed} certifies a sufficient condition, while the proportionality in \cref{eq:optimal_ratio} carries the predictive content for sample allocation.


\subsection{Proof of \Cref{thm:noise}}
\label{app:proof_noise}

\noisethm*

\begin{proof}
A trajectory $\tau = ((s_0, a_0), \ldots, (s_{H-1}, a_{H-1}))$ is sampled by executing $\pi_\theta$ in $\mathcal{M}$. Let $G_t$, $\hat{G}_t$, and $\hat{g}^{(1)}$ be as defined in \Cref{sec:noisy_rewards}.
Define the corresponding \reinforce estimator when the rewards are noise-free as
$g^{(1)} := \sum_{t=0}^{H-1} \nabla_\theta \log \pi_\theta(a_t \mid s_t)\, G_t$,
so $\mathbb{E}_{\tau}[g^{(1)}] = g_H$.

Let
$N_t := \sum_{t'=t}^{H-1} \gamma^{t'-t} \eta_{t'}$,
so that $\hat{G}_t = G_t + N_t$.
Let
$\delta^{(1)} := \sum_{t=0}^{H-1} \nabla_\theta \log \pi_\theta(a_t \mid s_t)\, N_t$,
so that $\hat{g}^{(1)} = g^{(1)} + \delta^{(1)}$.
Fix the sampled trajectory $\tau$ and take conditional expectation over the
reward noise. Because each $\eta_{t'}$ is independent of
the state-action history and has mean zero,
\begin{align*}
\mathbb{E}[N_t \mid \tau]
&= \sum_{t'=t}^{H-1} \gamma^{t'-t}\, \mathbb{E}[\eta_{t'} \mid \tau] \\
&= \sum_{t'=t}^{H-1} \gamma^{t'-t}\, \mathbb{E}[\eta_{t'}] \\
&= 0.
\end{align*}
Therefore
\begin{equation}
\label{eq:proof_noise_delta_conditional_mean_zero}
\begin{aligned}
\mathbb{E}[\delta^{(1)} \mid \tau]
&= \sum_{t=0}^{H-1} \nabla_\theta \log \pi_\theta(a_t \mid s_t)\,
   \mathbb{E}[N_t \mid \tau] \\
&= 0,
\end{aligned}
\end{equation}
and hence
\begin{equation*}
\mathbb{E}[\hat{g}^{(1)} \mid \tau] = g^{(1)}.
\end{equation*}
Therefore $\mathbb{E}[\hat{g}^{(1)}] = \mathbb{E}[g^{(1)}] = g_H$, and averaging
over $K$ trajectories gives $\mathbb{E}[\hat{g}] = g_H$.

Fix the sampled trajectory $\tau$ and define
\begin{equation*}
w_t := \nabla_\theta \log \pi_\theta(a_t \mid s_t).
\end{equation*}
Then
\begin{align*}
\delta^{(1)}
&= \sum_{t=0}^{H-1} w_t \sum_{t'=t}^{H-1} \gamma^{t'-t}\eta_{t'} \\
&= \sum_{t=0}^{H-1} \sum_{t'=t}^{H-1} \gamma^{t'-t} w_t \eta_{t'} \\
&= \sum_{t'=0}^{H-1} \sum_{\substack{t=0 \\ t \le t'}}^{H-1} \gamma^{t'-t} w_t \eta_{t'} \\
&= \sum_{t'=0}^{H-1} \sum_{t=0}^{t'} \gamma^{t'-t} w_t \eta_{t'}.
\end{align*}
Define
\begin{equation*}
v_{t'} := \sum_{t=0}^{t'} \gamma^{t'-t} w_t.
\end{equation*}
Then
\begin{equation*}
\delta^{(1)} = \sum_{t'=0}^{H-1} v_{t'} \eta_{t'}.
\end{equation*}
Using \Cref{eq:proof_noise_delta_conditional_mean_zero},
\begin{align*}
\mathrm{Var}[\delta^{(1)} \mid \tau]
&= \sum_{i=1}^d \mathrm{Var}[\delta^{(1)}_i \mid \tau] \\
&= \sum_{i=1}^d
   \mathbb{E}\!\left[
   \left(\delta^{(1)}_i-\mathbb{E}[\delta^{(1)}_i \mid \tau]\right)^2
   \middle| \tau\right] \\
&= \mathbb{E}\bigl[\|\delta^{(1)} - \mathbb{E}[\delta^{(1)} \mid \tau]\|^2
   \mid \tau\bigr] \\
&= \mathbb{E}\bigl[\|\delta^{(1)}\|^2 \mid \tau\bigr] \\
&= \mathbb{E}\!\left[\left\|\sum_{t'=0}^{H-1} v_{t'} \eta_{t'}\right\|^2
   \middle| \tau\right] \\
&= \mathbb{E}\!\left[\left\langle \sum_{t'=0}^{H-1} v_{t'} \eta_{t'},
   \sum_{u=0}^{H-1} v_u \eta_u \right\rangle \middle| \tau\right] \\
&= \mathbb{E}\!\left[\sum_{t'=0}^{H-1} \sum_{u=0}^{H-1}
   \left\langle v_{t'} \eta_{t'}, v_u \eta_u \right\rangle
   \middle| \tau\right] \\
&= \mathbb{E}\!\left[\sum_{t'=0}^{H-1} \sum_{u=0}^{H-1}
   \eta_{t'}\eta_u \langle v_{t'}, v_u\rangle \middle| \tau\right] \\
&= \sum_{t'=0}^{H-1} \sum_{u=0}^{H-1}
   \mathbb{E}[\eta_{t'}\eta_u \mid \tau] \, \langle v_{t'}, v_u\rangle \\
&= \sum_{t'=0}^{H-1}
   \mathbb{E}[\eta_{t'}^2 \mid \tau] \, \langle v_{t'}, v_{t'}\rangle + \sum_{t'=0}^{H-1} \sum_{\substack{u=0 \\ u \neq t'}}^{H-1}
   \mathbb{E}[\eta_{t'}\eta_u \mid \tau] \, \langle v_{t'}, v_u\rangle \\
\intertext{For $t' \neq u$, independence of the reward noise from $\tau$, together with
independence across time and zero mean, gives
$\mathbb{E}[\eta_{t'}\eta_u \mid \tau]
= \mathbb{E}[\eta_{t'}\eta_u]
= \mathbb{E}[\eta_{t'}]\mathbb{E}[\eta_u]
= 0$, while
$\mathbb{E}[\eta_{t'}^2 \mid \tau]
= \mathbb{E}[\eta_{t'}^2]
= \sigma_\eta^2$. Therefore}
&= \sum_{t'=0}^{H-1} \sigma_\eta^2 \|v_{t'}\|^2 \\
&= \sigma_\eta^2 \sum_{t'=0}^{H-1} \|v_{t'}\|^2.
\end{align*}
Next,
\begin{align*}
\|v_{t'}\|
&= \left\|\sum_{t=0}^{t'} \gamma^{t'-t} w_t\right\| \\
&\leq \sum_{t=0}^{t'} \gamma^{t'-t}\|w_t\| \\
&\leq \left(\sum_{t=0}^{t'} \gamma^{t'-t}\right)
   \max_{0 \leq t \leq H-1}\|w_t\| \\
&\leq \frac{1}{1-\gamma}\max_{0 \leq t \leq H-1}\|w_t\|.
\end{align*}
Substituting the bound on $\|v_{t'}\|$ into the conditional variance gives
\begin{align*}
\mathrm{Var}[\delta^{(1)} \mid \tau]
&\leq \sigma_\eta^2 \sum_{t'=0}^{H-1}
   \left(\frac{1}{1-\gamma}\max_{0 \leq t \leq H-1}\|w_t\|\right)^2 \\
&= \frac{\sigma_\eta^2 H}{(1-\gamma)^2}
   \max_{0 \leq t \leq H-1}\|w_t\|^2.
\end{align*}
Taking expectations over $\tau$ and using the definitions of $w_t$ and $W_H^2$
yields
\begin{align*}
\mathbb{E}[\mathrm{Var}[\delta^{(1)}|\tau]]
&\leq \frac{\sigma_\eta^2 H}{(1-\gamma)^2}
   \mathbb{E}\!\left[\max_{0 \leq t \leq H-1}\|w_t\|^2\right] \\
&= \frac{\sigma_\eta^2 H}{(1-\gamma)^2}
   \mathbb{E}\!\left[\max_{0 \leq t \leq H-1}
   \|\nabla_\theta \log \pi_\theta(a_t \mid s_t)\|^2\right] \\
&= \frac{\sigma_\eta^2 H W_H^2}{(1-\gamma)^2}.
\end{align*}
Now apply the law of total variance to each coordinate and sum over
coordinates. Here the conditional variance is over the reward noise given a
fixed trajectory $\tau$:
\begin{align}
\mathrm{Var}[\hat{g}^{(1)}]
&= \mathrm{Var}\!\bigl(\mathbb{E}[\hat{g}^{(1)} \mid \tau]\bigr)
   + \mathbb{E}\bigl[\mathrm{Var}[\hat{g}^{(1)} \mid \tau]\bigr] \notag \\
&= \mathrm{Var}[g^{(1)}]
   + \mathbb{E}\bigl[\mathrm{Var}[g^{(1)} + \delta^{(1)} \mid \tau]\bigr] \notag \\
\intertext{Because $g^{(1)}$ depends only on $\tau$, conditioning on $\tau$
makes $g^{(1)}$ deterministic. Hence
$\mathrm{Var}[g^{(1)} + \delta^{(1)} \mid \tau]
= \mathrm{Var}[\delta^{(1)} \mid \tau]$ almost surely, so}
&= \mathrm{Var}[g^{(1)}] + \mathbb{E}[\mathrm{Var}[\delta^{(1)} \mid \tau]] \notag \\
&\leq \mathrm{Var}[g^{(1)}]
   + \frac{\sigma_\eta^2 H W_H^2}{(1-\gamma)^2}.
\label{eq:proof_noise_single_trajectory_variance_bound}
\end{align}

Since the trajectories and their reward noises are sampled independently across
$k$, the estimators
$\hat{g}^{(1)},\ldots,\hat{g}^{(K)}$ are independent. Each estimator has mean
$g_H$. Write $\hat{g}_i$ and $(g_H)_i$ for the $i$th coordinates of $\hat{g}$
and $g_H$. Then
\begin{align}
\mathrm{Var}[\hat{g}]
&= \sum_{i=1}^d \mathrm{Var}[\hat{g}_i] \notag \\
&= \sum_{i=1}^d
   \mathbb{E}\!\left[(\hat{g}_i-\mathbb{E}[\hat{g}_i])^2\right] \notag \\
&= \sum_{i=1}^d
   \mathbb{E}\!\left[(\hat{g}_i-(g_H)_i)^2\right] \notag \\
&= \mathbb{E}\!\left[\sum_{i=1}^d(\hat{g}_i-(g_H)_i)^2\right] \notag \\
&= \mathbb{E}\!\left[\|\hat{g}-g_H\|^2\right] \notag \\
&= \mathbb{E}\!\left[
   \left\|\frac{1}{K}\sum_{k=1}^K \hat{g}^{(k)}-g_H\right\|^2\right] \notag \\
&= \frac{1}{K^2}\mathbb{E}\!\left[
   \left\|\sum_{k=1}^K(\hat{g}^{(k)}-g_H)\right\|^2\right] \notag \\
&= \frac{1}{K^2}\mathbb{E}\!\left[
   \left\langle \sum_{k=1}^K(\hat{g}^{(k)}-g_H),
   \sum_{\ell=1}^K(\hat{g}^{(\ell)}-g_H) \right\rangle
   \right] \notag \\
&= \frac{1}{K^2}\sum_{k=1}^K \sum_{\ell=1}^K
   \mathbb{E}\!\left[
   \left\langle \hat{g}^{(k)}-g_H,\hat{g}^{(\ell)}-g_H \right\rangle
   \right] \notag \\
\intertext{For $k \neq \ell$, independence gives
$\mathbb{E}[\langle \hat{g}^{(k)}-g_H,\hat{g}^{(\ell)}-g_H\rangle]
= \langle \mathbb{E}[\hat{g}^{(k)}-g_H],
   \mathbb{E}[\hat{g}^{(\ell)}-g_H]\rangle
= 0$. Therefore only the terms with $k = \ell$ remain, so}
&= \frac{1}{K^2}\sum_{k=1}^K
   \mathbb{E}\!\left[
   \left\langle \hat{g}^{(k)}-g_H,\hat{g}^{(k)}-g_H \right\rangle
   \right] \notag \\
&= \frac{1}{K^2}\sum_{k=1}^K
   \mathbb{E}\!\left[
   \sum_{i=1}^d(\hat{g}^{(k)}_i-(g_H)_i)^2\right] \notag \\
&= \frac{1}{K^2}\sum_{k=1}^K \sum_{i=1}^d
   \mathbb{E}\!\left[(\hat{g}^{(k)}_i-(g_H)_i)^2\right] \notag \\
\intertext{Since each coordinate of $\hat{g}^{(k)}$ has mean the corresponding
coordinate of $g_H$,}
&= \frac{1}{K^2}\sum_{k=1}^K \sum_{i=1}^d
   \mathrm{Var}[\hat{g}^{(k)}_i] \notag \\
\intertext{By the definition of $\mathrm{Var}$ for vector-valued estimators,}
&= \frac{1}{K^2}\sum_{k=1}^K \mathrm{Var}[\hat{g}^{(k)}] \notag \\
\intertext{The estimators $\hat{g}^{(1)},\ldots,\hat{g}^{(K)}$ have the same
distribution, so each term in the sum equals $\mathrm{Var}[\hat{g}^{(1)}]$:}
&= \frac{1}{K^2}\sum_{k=1}^K \mathrm{Var}[\hat{g}^{(1)}] \notag \\
&= \frac{1}{K}\mathrm{Var}[\hat{g}^{(1)}].
\label{eq:proof_noise_noisy_average_variance}
\end{align}
When the rewards are noise-free, the averaged estimator $\hat{g}$ equals
$\hat{g} = \frac{1}{K}\sum_{k=1}^K g^{(k)}$. Repeating the expansion above gives
\begin{align}
\mathrm{Var}[\hat{g}]_{\eta \equiv 0}
&= \mathrm{Var}\!\left[\frac{1}{K}\sum_{k=1}^K g^{(k)}\right] \notag \\
&= \frac{1}{K^2}\sum_{k=1}^K \sum_{\ell=1}^K
   \mathbb{E}\!\left[
   \left\langle g^{(k)}-g_H,g^{(\ell)}-g_H \right\rangle
   \right] \notag \\
\intertext{For $k \neq \ell$, independence gives
$\mathbb{E}[\langle g^{(k)}-g_H,g^{(\ell)}-g_H\rangle]
= \langle \mathbb{E}[g^{(k)}-g_H],\mathbb{E}[g^{(\ell)}-g_H]\rangle
= 0$, so}
&= \frac{1}{K^2}\sum_{k=1}^K
   \mathbb{E}\!\left[\|g^{(k)}-g_H\|^2\right] \notag \\
&= \frac{1}{K^2}\sum_{k=1}^K \mathrm{Var}[g^{(k)}] \notag \\
&= \frac{1}{K^2}\sum_{k=1}^K \mathrm{Var}[g^{(1)}] \notag \\
&= \frac{1}{K}\mathrm{Var}[g^{(1)}].
\label{eq:proof_noise_noise_free_average_variance}
\end{align}
Combining
\Cref{eq:proof_noise_noisy_average_variance} with
\Cref{eq:proof_noise_single_trajectory_variance_bound} gives
\begin{align*}
\mathrm{Var}[\hat{g}]
&= \frac{1}{K}\mathrm{Var}[\hat{g}^{(1)}] \\
&\leq \frac{1}{K}\left(\mathrm{Var}[g^{(1)}]
   + \frac{\sigma_\eta^2 H W_H^2}{(1-\gamma)^2}\right) \\
&= \frac{1}{K}\mathrm{Var}[g^{(1)}]
   + \frac{\sigma_\eta^2 H W_H^2}{K(1-\gamma)^2} \\
&= \mathrm{Var}[\hat{g}]_{\eta \equiv 0}
   + \frac{\sigma_\eta^2 H W_H^2}{K(1-\gamma)^2},
\end{align*}
where the last equality uses
\Cref{eq:proof_noise_noise_free_average_variance}.
\end{proof}

\subsection{Visualization of the three fidelity-cost regimes from \Cref{cor:fidelity}}
\label{app:fidelity_visualizations}

\Cref{fig:fidelity_examples} plots $\sigma_\eta^2(c)$ and $\Phi(c) = c\,\sigma_\eta^2(c)$ for the three regimes analyzed in \Cref{sec:noisy_rewards}.

\begin{figure}[htb]
  \centering
  \begin{subfigure}[t]{0.28\linewidth}
    \centering
    \includegraphics[width=\linewidth]{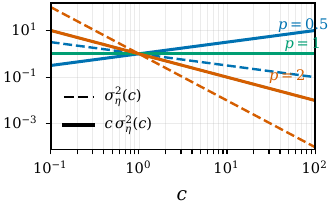}
    \caption{Power-law fidelity.}
    \label{fig:fidelity_powerlaw}
  \end{subfigure}\hfill
  \begin{subfigure}[t]{0.28\linewidth}
    \centering
    \includegraphics[width=\linewidth]{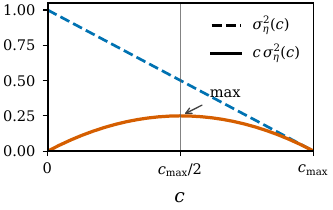}
    \caption{Bounded fidelity.}
    \label{fig:fidelity_bounded}
  \end{subfigure}\hfill
  \begin{subfigure}[t]{0.28\linewidth}
    \centering
    \includegraphics[width=\linewidth]{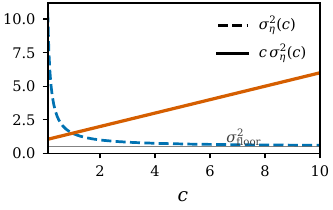}
    \caption{Irreducible noise floor.}
    \label{fig:fidelity_floor}
  \end{subfigure}
  \caption{Three examples of how the noise variance $\sigma_\eta^2(c)$ depends on the cost $c$, and the corresponding $\Phi(c) = c\,\sigma_\eta^2(c)$ from \Cref{cor:fidelity}.}
  \label{fig:fidelity_examples}
\end{figure}

\subsection[Proof of the biased reward proposition]{Proof of \Cref{prop:biased_reward}}
\label{app:proof_biased_reward}

\biasedrewardthm*

\begin{proof}
For one trajectory $\tau = ((s_0, a_0), \ldots, (s_{H-1}, a_{H-1}))$ sampled
by executing $\pi_\theta$ under the true dynamics $f$, define the biased
finite-horizon return from time $t$ onward by
\begin{equation*}
\tilde{G}_t := \sum_{t'=t}^{H-1}\gamma^{t'-t}\tilde{r}(s_{t'},a_{t'}).
\end{equation*}
The single-trajectory estimator using biased rewards is
\begin{equation*}
\tilde{g}^{(1)}
:= \sum_{t=0}^{H-1}
   \nabla_\theta \log \pi_\theta(a_t \mid s_t)\,\tilde{G}_t.
\end{equation*}

The estimator $\tilde{g}^{(1)}$ is the finite-horizon \reinforce estimator for
the \mdp $\tilde{\mathcal{M}}$, which has the true dynamics $f$ and the biased
reward $\tilde{r}$. Therefore the finite-horizon policy-gradient identity gives
\begin{equation*}
\mathbb{E}[\tilde{g}^{(1)}]
= \nabla_\theta J_H(\pi_\theta,\tilde{\mathcal{M}}).
\end{equation*}
Averaging over $K$ independent trajectories does not change the mean:
\begin{align*}
\mathbb{E}[\tilde{g}]
&= \mathbb{E}\!\left[\frac{1}{K}\sum_{k=1}^K \tilde{g}^{(k)}\right] \\
&= \frac{1}{K}\sum_{k=1}^K \mathbb{E}[\tilde{g}^{(k)}] \\
&= \nabla_\theta J_H(\pi_\theta,\tilde{\mathcal{M}}).
\end{align*}
By the definitions of $\tilde{r}$, $J_H$, and $B_H$,
\begin{align*}
J_H(\pi_\theta,\tilde{\mathcal{M}})
&= \mathbb{E}\!\left[
   \sum_{t=0}^{H-1}\gamma^t \tilde{r}(s_t,a_t)\right] \\
&= \mathbb{E}\!\left[
   \sum_{t=0}^{H-1}\gamma^t (r(s_t,a_t) + b(s_t,a_t))\right] \\
&= \mathbb{E}\!\left[
   \sum_{t=0}^{H-1}\gamma^t r(s_t,a_t)\right]
   + \mathbb{E}\!\left[
   \sum_{t=0}^{H-1}\gamma^t b(s_t,a_t)\right] \\
&= J_H(\pi_\theta,\mathcal{M}) + B_H(\theta).
\end{align*}
Taking gradients with respect to $\theta$ gives
\begin{align*}
\nabla_\theta J_H(\pi_\theta,\tilde{\mathcal{M}})
&= \nabla_\theta J_H(\pi_\theta,\mathcal{M})
   + \nabla_\theta B_H(\theta) \\
&= g_H + \nabla_\theta B_H(\theta).
\end{align*}
Combining the preceding displays proves \Cref{eq:biased_reward_mean}.

Let
\begin{equation*}
\mu_b := \mathbb{E}[\tilde{g}] - g_H.
\end{equation*}
By \Cref{eq:biased_reward_mean}, $\mu_b = \nabla_\theta B_H(\theta)$.
We use the identity $\|x+y\|^2=\|x\|^2+2\langle x,y\rangle+\|y\|^2$
with $x=\tilde{g}-\mathbb{E}[\tilde{g}]$ and $y=\mu_b$.
We also use linearity of expectation and the fact that $\mu_b$ is a fixed
vector once $\theta$ is fixed, so
$\mathbb{E}[\|\mu_b\|^2]=\|\mu_b\|^2$.
Then
\begin{align*}
\mathbb{E}\!\left[\|\tilde{g}-g_H\|^2\right]
&= \mathbb{E}\!\left[
   \|\tilde{g}-\mathbb{E}[\tilde{g}] + \mathbb{E}[\tilde{g}]-g_H\|^2
   \right] \\
&= \mathbb{E}\!\left[
   \|\tilde{g}-\mathbb{E}[\tilde{g}] + \mu_b\|^2
   \right] \\
&= \mathbb{E}\!\left[
   \|\tilde{g}-\mathbb{E}[\tilde{g}]\|^2
   + 2\left\langle \tilde{g}-\mathbb{E}[\tilde{g}],\mu_b\right\rangle
   + \|\mu_b\|^2
   \right] \\
&= \mathbb{E}\!\left[\|\tilde{g}-\mathbb{E}[\tilde{g}]\|^2\right]
   + 2\,\mathbb{E}\!\left[
   \left\langle \tilde{g}-\mathbb{E}[\tilde{g}],\mu_b\right\rangle
   \right]
   + \|\mu_b\|^2 \\
&= \mathrm{Var}[\tilde{g}] + \|\mu_b\|^2
\end{align*}
because $\mathbb{E}[\langle \tilde{g}-\mathbb{E}[\tilde{g}],\mu_b\rangle]
= \langle \mathbb{E}[\tilde{g}]-\mathbb{E}[\tilde{g}],\mu_b\rangle = 0$.
Since the estimators
$\tilde{g}^{(1)},\ldots,\tilde{g}^{(K)}$ are independent and identically
distributed,
\begin{align*}
\mathrm{Var}[\tilde{g}]
&= \mathrm{Var}\!\left[\frac{1}{K}\sum_{k=1}^K \tilde{g}^{(k)}\right] \\
&= \frac{1}{K^2}\sum_{k=1}^K \mathrm{Var}[\tilde{g}^{(k)}] \\
&= \frac{1}{K}\mathrm{Var}[\tilde{g}^{(1)}].
\end{align*}
Substituting $\mu_b = \nabla_\theta B_H(\theta)$ gives
\Cref{eq:biased_reward_mse}. If $\nabla_\theta B_H(\theta) \neq 0$, the second
term in \Cref{eq:biased_reward_mse} is positive and does not depend on $K$.
\end{proof}

\section{Experimental details}

\subsection{Synthetic teacher environment}
\label{app:teacher_env}

The environment is a synthetic continuous-control environment whose dynamics and reward function are defined by frozen, randomly-initialized neural networks. Decoupling the environment from any particular physics simulator gives exact control over dynamics complexity, reward fidelity, and partial observability, which enables the controlled scaling experiment of \Cref{sec:error_scaling}.

The teacher consists of two 2-layer ReLU MLPs.
The dynamics teacher $f_{\mathrm{dyn}}: \mathbb{R}^{d_s + d_a} \to \mathbb{R}^{d_s}$ maps $(s_t, a_t)$ to the next state via $s_{t+1} = \tanh(f_{\mathrm{dyn}}([s_t, a_t]))$.
The reward teacher $f_{\mathrm{rew}}: \mathbb{R}^{d_s + d_a} \to \mathbb{R}$ produces the scalar reward $r_t = f_{\mathrm{rew}}([s_t, a_t]) / \sqrt{d_h}$, where $d_h$ is the teacher hidden width.
All teacher weights are drawn from $\mathcal{N}(0, 1)$ using a fixed seed and never updated, so the same seed yields the identical \mdp.
The $\tanh$ activation in dynamics constrains the state to $[-1, 1]^{d_s}$, and the $1/\sqrt{d_h}$ reward scaling ensures comparable reward magnitude across teacher widths.
Default dimensions: $d_s = 12$, $d_a = 4$, $d_h = 64$, episode length $T = 500$.

\subsection{Experiment details: power-law scaling of dynamics and reward error}
\label{app:error_scaling_details}

This appendix collects the protocol and engineering details for the error--sample-size scaling experiment in \Cref{sec:error_scaling}, whose main-text result is summarized in \Cref{fig:error_scaling}.

We use the teacher of \Cref{app:teacher_env}, which keeps the experiment inside the regime where the assumptions of \Cref{thm:optimal_split} hold.
Training transitions are drawn i.i.d.\ from the teacher's induced state distribution by rolling out a smooth reference policy under the frozen teacher and subsampling $(s, a, s', r)$ tuples; the released code records the exact rollout and subsampling procedure.

Power-law exponents are fit on the seven anchors $N \in \{2{,}000,\; 5{,}000,\; 10{,}000,\; 20{,}000,\; 50{,}000,\; 100{,}000,\; 200{,}000\}$, each with 100 independent seeds.
A single $N = 500{,}000$ run is preserved in the released dataset for future analyses but is excluded from both the figure and the fit, because only one seed was completed at that anchor.

For each $(N, \text{seed})$ we train two independent students that mirror the teacher architecture: a dynamics head $f_{\mathrm{dyn}}^{\theta}: \mathbb{R}^{d_s + d_a} \to \mathbb{R}^{d_s}$ and a reward head $f_{\mathrm{rew}}^{\phi}: \mathbb{R}^{d_s + d_a} \to \mathbb{R}$, each implemented as a 2-layer ReLU MLP with hidden width $d_h = 64$ matching the teacher (no output activation, no shared trunk, no weight tying, no input/target normalization, default PyTorch \texttt{nn.Linear} initialization).
Both heads are optimized independently with Adam at learning rate $10^{-3}$ (PyTorch defaults for $\beta_1, \beta_2, \epsilon$; no weight decay) and a mini-batch size of $256$ drawn from the same $N$-sample training pool with reshuffling each epoch.
Training proceeds for a fixed schedule of $200$ epochs with no learning-rate schedule and no early stopping; the held-out validation MSE is recorded every epoch but is not used to gate the optimizer, and the reported $\edyn(\Ndyn)$ and $\erew(\Nrew)$ are the final-epoch validation MSEs of the two heads, computed under \texttt{torch.no\_grad}.
The training objective is the per-head mean-squared error, $\mathcal{L}_{\mathrm{dyn}}(\theta) = \tfrac{1}{|\mathcal{B}|} \sum_{(s, a, s') \in \mathcal{B}} \lVert f_{\mathrm{dyn}}^{\theta}([s, a]) - s' \rVert_2^2$ and $\mathcal{L}_{\mathrm{rew}}(\phi) = \tfrac{1}{|\mathcal{B}|} \sum_{(s, a, r) \in \mathcal{B}} (f_{\mathrm{rew}}^{\phi}([s, a]) - r)^2$, evaluated on each mini-batch $\mathcal{B}$.
At $200$ training epochs, and for every anchor $N \ge 2{,}000$, the training MSE is fully saturated, so the final-epoch values used by the fit coincide with the converged held-out loss to within seed noise.

Each anchor evaluates against a fixed held-out set of $5{,}000$ transitions, drawn independently of the training pool.
Holding the validation set size constant across $N$ keeps the evaluation noise floor independent of the training pool size, so any $N$-dependence of the measured MSE reflects student generalization rather than a changing estimator variance.

Per anchor, we report the across-seed mean MSE in the figure with $\pm 1$ standard error of the mean (SEM) error bars; this across-seed SEM is distinct from the bootstrap standard errors reported below for the fitted parameters.
The power laws $\edyn(\Ndyn) = A_d \cdot \Ndyn^{-\alpha}$ and $\erew(\Nrew) = A_r \cdot \Nrew^{-\beta}$ are fit by ordinary log--log linear regression on the per-anchor means, where $\alpha$ is the dynamics exponent and $\beta$ is the reward exponent reported in \Cref{sec:error_scaling}.
Uncertainty on $(A_d, A_r, \alpha, \beta)$ is obtained by a stratified bootstrap with $1{,}000$ resamples that draws seeds with replacement \emph{within} each anchor before refitting; we report the bootstrap standard error and the $2.5$/$97.5$ percentile interval as the 95\% confidence interval.
The resulting bootstrap standard errors are $A_d \pm 0.04$, $\alpha \pm 0.01$, $A_r \pm 13.3$, and $\beta \pm 0.02$, and the 95\% bootstrap confidence intervals are $A_d \in [0.28,\, 0.42]$, $\alpha \in [0.09,\, 0.13]$, $A_r \in [68.1,\, 121]$, and $\beta \in [0.93,\, 0.99]$.


\end{document}